\PassOptionsToPackage{sort}{natbib}
\documentclass{article}

\usepackage[preprint]{neurips_2026}
\usepackage[utf8]{inputenc}
\usepackage[T1]{fontenc}
\usepackage{hyperref}
\usepackage{url}
\usepackage{booktabs}
\usepackage{amsfonts}
\usepackage{amsmath,amssymb,amsthm,mathtools}
\usepackage{nicefrac}
\usepackage{microtype}
\usepackage{xcolor}
\usepackage{graphicx}
\usepackage{subcaption}
\usepackage{multirow}
\usepackage{array}
\usepackage{enumitem}
\usepackage{arydshln}
\usepackage{csquotes}
\usepackage{wrapfig}
\usepackage{nicefrac}
\usepackage{siunitx}

\setcitestyle{numbers,square}
\renewcommand{\paragraph}[1]{\textbf{#1}}

\newtheorem{theorem}{Theorem}
\newtheorem{lemma}[theorem]{Lemma}

\newtheorem{corollary}[theorem]{Corollary}
\theoremstyle{definition}

\providecommand{\R}{\mathbb{R}}
\providecommand{\eps}{\varepsilon}

\providecommand{\Cstrict}{\mathcal{C}}
\providecommand{\Cbar}{\overline{\mathcal{C}}}

\providecommand{\esssup}{\operatorname*{ess\,sup}}

\newcommand{\E}{\mathbb{E}}
\newcommand{\Prob}{\mathbb{P}}

\newcommand{\Rstar}{R^{\star}}
\newcommand{\RPO}{R_{\mathrm{PO2}}}

\title{Grid Games: The Power of Multiple Grids \\ for Quantizing Large Language Models}

\author{Vage Egiazarian\thanks{Equal contributions.}\\
ISTA \\ 
\And
Erik Schultheis\footnotemark[1] \\
ISTA \\
\And
Andrei Panferov \\
ISTA \\
\AND
Earl Killian\\
Aril Computer\\
\And
Torsten Hoefler\\
ETH Zürich\\
\And
Dan Alistarh\thanks{Correspondence to: \texttt{dan.alistarh@ista.ac.at}.}\\
ISTA \& Red Hat AI\\
}

\begin{document}
\vspace{-1em}
\maketitle
\vspace{-1em}
\begin{abstract}
A major recent advance in quantization is given by microscaled 4-bit formats such as NVFP4 and MXFP4, quantizing values into small groups sharing a scale, assuming a fixed floating-point grid. 
In this paper, we study the following natural extension: assume that, for each group of values, we are free to select the ``better'' among two or more 4-bit grids marked by one or more bits in the scale value.
We formalize the \textit{power-of-two-grids (PO2)} problem, and provide theoretical results showing that practical small-group formats such as MXFP or NVFP can benefit significantly from PO2 grids, while the advantage vanishes for very large groups. 
On the practical side, we instantiate several grid families, including 1) PO2(NF4), which pairs the standard NF4 normal grid with a learned grid, 2) MPO2, a grid pair that is fully learned over real weights and activations, 3) PO2(Split87), an explicit-zero asymmetric grid and 4) SFP4, a TensorCore-implementable triple which pairs NVFP4 with two shifted variants.  
Results for post-training quantization of standard open models and pre-training of Llama-like models show that adaptive grids consistently improve accuracy vs single-grid FP4 under both weight-only and weight+activation. Source code is available at \href{https://github.com/IST-DASLab/GridGames}{github}.
\end{abstract}

\vspace{-1em}
\section{Introduction}
\label{sec:intro}
\vspace{-0.5em}

Quantization of model states, especially weights and activations, has become a standard way to reduce the serving cost of large language models (LLMs)~\citep{nagel2021white,frantar2023gptq,lin2024awq,dettmers2023qlora}.
The recent addition of hardware support for microscaled 4-bit precision datatypes such as NVFP4 and MXFP4 makes quantization attractive not only for inference, but also for training~\citep{rouhani2023microscaling,khailany2025fp4,castro2025quartet,panferov2026quartetii}.
In these formats, a small group of tensor values, typically 16 or 32, shares one scale; normalized values are then represented by a fixed low-bit codebook.
Thus, a key design question is which grid should receive the scarce 4-bit budget.\looseness=-1

Most existing methods focus on 4-bit precision due to the good compression-accuracy trade-off, and assume that the grid is fixed in advance.
Examples are the INT4 uniform grid and the FP4 E2M1 floating-point grid, whose normalized levels are denser near zero.
The NF4 grid uses a quantile-inspired representation targeting Gaussian weights~\citep{dettmers2023qlora}, while Lloyd--Max and learned-codebook methods fit a grid to a calibration distribution~\citep{lloyd1982least,max1960quantizing,malinovskii2025higgs}.
Orthogonal approaches seek to improve rounding error: GPTQ applies Hessian-aware layerwise corrections~\citep{frantar2023gptq}, AWQ protects activation-salient weights~\citep{lin2024awq}, and QuaRot/QuIP\# use incoherence transforms~\citep{ashkboos2024quarot,chee2024quip}.
Here, we focus on an orthogonal question: what if, when encoding a group of values, we can pick the best option among a small set of fixed grids?


In this paper, we investigate this idea from theory to kernels.
A first result is that, if groups are large, the benefit of per-block grids  disappears: every block sees essentially the same normalized distribution, so a single well-chosen grid is enough. 
However, for the small (e.g., 16 or 32) block sizes used in practice, normalized groups are heterogeneous: one could have \emph{spiky} groups (containing an outlier), or \emph{flat} groups, with several entries of comparable magnitude. 
A specialized dual-codebook quantizer, choosing between  grids based on group statistics, can leverage this dichotomy. 

\paragraph{Contributions.}
We provide a systematic investigation of group-adaptive grid choices, establishing the first theoretical treatment of this approach, and providing multiple practical findings: 
\begin{enumerate}[topsep=0pt,itemsep=2pt,leftmargin=1.5em]
    \item We formalize the problem of multiple-codebook blockwise (group) quantization. Although it is NP-hard in general, in the large-group limit, the advantage of having multiple grids vanishes. 
    \item Yet, in standard microscaling settings, multiple grid choices can be effective. We provide a \emph{competitive analysis} framework that relates the loss achievable via a grid family over a distribution to the \textit{best single learnable grid} over that distribution, via a concavity argument. For instance, we show that choosing the best between INT4 and FP4 per group (IF4~\citep{cook2026adaptiveblockscaled}) is useful, but sub-optimal: this grid pair has, roughly, 10\% worse error than the best single learned grid for the model. 
    \item We define and investigate several \textit{power-of-two (PO2) grid families}: MPO2, a pair of grids optimized over pooled weight and activation groups; PO2(NF4), which pairs the standard NF4 Normal grid~\citep{dettmers2023qlora} with a residual grid; PO2(Split87), an explicit-zero asymmetric grid we developed.
    \item  Our main practical contribution is \textit{SFP4 (Shifted NVFP4)}, a grid triple which keeps the standard NVFP4 E2M1 codebook, but adds \textit{two} optimally-shifted NVFP4  as grid options. We show that matmuls using SFP4 can be decoded correctly using NVIDIA Blackwell TensorCores, while being accuracy-competitive with general learned grids or even some grid pair constructions.   
    \item We provide extensive evaluations across multiple benchmarks: 1) fitting standard distributions; 2) post-training weight-and-activation quantization of LLMs, and 3) pre-training experiments of small language models. 
    
\end{enumerate}

Our results show that certain grid pairs, such as the PO2(NF4) pair that encodes a Normal-distribution prior plus a residual,  can \emph{surpass} the accuracy of the best individually-learned grid for the model or layer, due to their adaptive nature and the practical validity of the Normal prior. This establishes the accuracy boost of the PO2 approach, and suggests it is an alternative to current fixed-grid approaches. 

\vspace{-1em}
\section{Related Work}
\label{sec:related}

\paragraph{Scalar and learned grids.} 
Uniform integer grids (e.g., INT4) and floating-point grids (e.g., FP4) are the default hardware choices.
The NF4 grid was proposed as an alternative as part of QLoRA~\citep{dettmers2023qlora}; its main idea is to choose grid points using a normal-distribution prior.
Other grid constructions~\citep{malinovskii2025higgs} used Lloyd--Max quantizers to optimize scalar MSE for a target distribution.
Our \textit{dual-grid} construction is complementary to single-grid choices: each constituent grid can be either optimized w.r.t. a target distribution, or learned over real-world tensors. 
The recent \textbf{BOF4} work~\citep{blumenberg2025bof4} derives Lloyd/EM-style updates for codebooks that minimize MAE or MSE in the original, unnormalized weight space. Their signed-normalization variant, BOF4-S, can be viewed as a  special case of our dual-grid formulation with grid pair $(G,-G)$, where the selector is not chosen to minimize the error, but is fixed deterministically by the sign of the maximum-magnitude entry. As we show, BOF4 is competitive with simple adaptive pairs such as IF4, but can lose to general grid pairs. 
More broadly, \textbf{LO-BCQ}~\citep{elangovan2025lobcq} decomposes tensors into blocks, assigns each block to one of several learned scalar codebooks, and stores a per-block selector. 
One key difference is in the codebook constraints: LO-BCQ learns up to 16 arbitrary 16-entry codebooks; this yields good PTQ accuracy but induces selector overhead and a decode path for learned codebooks. 

\paragraph{IF4.}
Concurrent work~\citep{cook2026adaptiveblockscaled} proposes a special case of per-block grid selection, choosing between INT4 and FP4-like grids, recording the chosen grid as metadata in the group scale.
Our theory explains both the strengths and limitations of this approach: 
although it improves on NVFP4, IF4 remains inferior to statically-optimal trained grids, both in theory and in practice, and is roughly on par with our SFP4 construction, which has the advantage of being implementable on current hardware.  

\paragraph{Quantization methods.}
Methods like GPTQ~\citep{frantar2023gptq}, AWQ~\citep{lin2024awq}, and SqueezeLLM~\citep{kim2024squeezellm} aim to reduce the accuracy effects of quantization by changing weights, scales, or basis before rounding.
Here, we focus on the \textit{capacity} of the codebooks, which is complementary.
Specifically, GPTQ composes well with dual-grid quantization, two-grid quantization can benefit from Hadamard transforms. 

\paragraph{Quantized training.}
Another related direction is FP4 training~\citep{chmiel2025fp4wayfullyquantized,castro2025quartet,panferov2026quartetii,tseng2025training}. 
Four-Over-Six~\citep{cook2025fouroversix} modifies the scale rule for FP4 blocks, searching over a small set of divisors rather than always setting the scale to $\max/6$.
This is orthogonal to grid selection: they change the scale magnitude, while dual-grid methods change the set of normalized representable values.
As we show, some of our proposals are compatible with scaling methods such as Four-Over-Six.

\section{PO2: The Power of Two Grids for Quantization}

\subsection{Problem Setup}
\label{sec:setup}

We consider scalar quantization applied independently to blocks of size $g$.
A block is written as $X=(X_1,\ldots,X_g)\in\mathbb{R}^g$,
where the coordinates are i.i.d. real-valued random variables.  Let us define
        $Z_i:=|X_i|, \text{ and }
        M_g:=\max_{1\le i\le g} Z_i.$
The quantizer uses absmax normalization: if $M_g>0$, define
$
        A_{i,g}:=Z_i/M_g\in[0,1].
$
If $M_g=0$, the whole block is zero; in this case we set $A_{i,g}=0$ and 
its quantization error is zero.  
To simplify notation, assume w.l.o.g. on the positive half of an eight-level signed grid, and define the (closed) set of grids as: 
\[
        {\mathcal C}
        :=\{B:0=b_0\le b_1\le\cdots\le b_6\le b_7=1\} \,.
\]  
For $a\in[0,1]$, let
$
        \psi_B(a):=\min_{0\le j\le 7}(a-b_j)^2.
$
In other words, $\psi_B(a)$ is the squared error incurred by rounding the
normalized magnitude $a$ to its nearest grid point in $B$.  The corresponding
signed absmax quantization loss for a block is
\[
        L_B^{(g)}(X)
        := M_g^2\,\frac1g\sum_{i=1}^g \psi_B(A_{i,g}).
\]
This formula is the only property of the quantizer needed in the theorem below.
Further, we define the \textit{best single-grid risk} and adaptive two-grid, or power-of-two-grids (PO2), risks, respectively, as
\[
        R_g^\star
        :=\min_{B\in\overline{\mathcal C}}\mathbb{E}L_B^{(g)}(X) \textnormal{ \quad and \quad } R_{\mathrm{PO2},g}^\star
        :=\inf_{B_1,B_2\in\overline{\mathcal C}}
          \mathbb{E}\!\left[\min\{L_{B_1}^{(g)}(X),L_{B_2}^{(g)}(X)\}\right].
\]

\subsection{Theoretical Analysis}
\label{sec:theoretical-analysis}

Our first result shows that the advantage of selecting
between two fixed absmax-normalized grids disappears as the block size grows to infinity. It appears intuitive following the Glivenko–Cantelli and Central Limit theorems. The full argument is in Appendix~\ref{app:proofs}.

\begin{theorem}[No asymptotic advantage]
\label{thm:no-asymptotic-advantage}
Let \(X_1,X_2,\ldots\) be i.i.d. real-valued random variables, and set
\(Z_i=|X_i|\).  Define $
        x_{\max}:=\operatorname*{ess\,sup} Z_1
        =
        \inf\{x\in\mathbb R:\mathbb P(Z_1\le x)=1\}
        $ and assume that $x_{\max}$ is finite. 
Equivalently,
        $\mathbb P(Z_1\le x_{\max})=1$
and, for every \(\varepsilon>0\),
$
        \mathbb P(Z_1>x_{\max}-\varepsilon)>0 .
$
Then  
\[
        \lim_{g\to\infty}
        \bigl(R_g^\star-R_{\mathrm{PO2},g}^\star\bigr)=0.
\]
\end{theorem}

\paragraph{Proof sketch.}
The assumption on the supremum implies that the block
maximum $M_g=\max_{i\le g}|X_i|$ converges almost surely to $x_{\max}$.  
Given this, the main idea in the proof is that, for a fixed grid $B$, 
the average normalized loss
$g^{-1}\sum_i\psi_B(Y_i)$ converges by the law of large numbers to
$\mathbb{E}\psi_B(Y_1)$.  Since the grid space ${\mathcal C}$ is
compact and the $\psi_B$ vary Lipschitz-continuous, this convergence is uniform over all grids.  Hence
$L_B^{(g)}$ converges uniformly over $B$ to the deterministic risk
$
        \ell_\infty(B):=x_{\max}^2\,\mathbb{E}\psi_B(|X_1|/x_{\max}).
$
Consequently, the best single-grid risk converges to
$\inf_B\ell_\infty(B)$.  The two-grid objective converges to
\[
        \inf_{B_1,B_2}\min\{\ell_\infty(B_1),\ell_\infty(B_2)\}
        =\inf_B\ell_\infty(B),
\]
because, in the limit, every large block has the same deterministic normalized
loss profile.  Thus per-block selection between two fixed grids cannot improve
on the best single fixed grid asymptotically.

\paragraph{The Interesting Bimodal Case.} The prior result suggests that, for two-grid choices to matter, we need to be in the small-group-size, microscaling regime. 
 Let
\[
\mu(X)=\frac1g\sum_i \frac{|X_i|}{M_g}
\]
be the average normalized magnitude of a group.
Two grids help only when block statistics vary from group to group. For example, small $\mu$ would indicate a \textit{spiky} block, dominated by one or more outliers; large $\mu$ indicates a \textit{flat} block with most entries close to the average.

\begin{wraptable}{r}{0.6\textwidth}
\centering
\caption{Student-$t$ fit to real LLM weight distributions.}
\label{tab:weight_fit}
\begin{tabular}{lcccc}
\toprule
Model & MLE $\nu$ & Kurtosis & Best QQ fit \\
\midrule
Qwen3-0.6B & 6.65 & 2.55 & $t_7$ ($r=0.9999$) \\
Qwen3-8B   & 10.98 & 1.17 & $\nicefrac{t_7}{t_{10}}$ ($r=0.9967$) \\
\bottomrule
\end{tabular}
\vspace{-\intextsep}
\end{wraptable}

Real transformer weights support this bi-modal view. As shown in Table~\ref{tab:weight_fit}, for Qwen3-0.6B and Qwen3-8B, global Student-$t$ fits have $\nu$ in the  range 7--11 and per-layer kurtosis varies widely, consistent with prior heavy-tail analyses of neural-network weights~\citep{martin2021implicit,hodgkinson2021multiplicative}.

\paragraph{Concavity.} Let $S$ and $F$ be a partition of the space of groups, e.g., into spiky and flat blocks for a fixed $\mu$ threshold, and let $p=\Prob(S)$.
For a grid family $\mathcal{G}$, the best risk achievable by any single grid when the spiky fraction is $p$ is
$V_{\mathcal{G}}(p):=\inf_{B\in\mathcal{G}}\left(pR_S(B)+(1-p)R_F(B)\right).$ Then, notice that: 

\begin{lemma}[Concavity]
\label{lem:concavity}
The function $V_{\mathcal{G}}(p)$ is concave in $p$.
\end{lemma}

Concavity is the basic reason why two specialized grids can beat a single one: a single grid lies on the concave envelope of mixtures, while adaptive selection pays the two conditional optima separately. 
Consequently, even finding a grid pair where each grid  \textit{approximates} the optima across the two group splits (e.g., finding a good grid for spiky groups and a good one for  flat ones) can lead to competitive results globally, relative to the best single grid for the distribution. This is our  main motivating result:   

\begin{corollary}
\label{thm:competitive}
Let $R_S^\star=\inf_B R_S(B)$ and $R_F^\star=\inf_B R_F(B)$.
If $R_S(B_S)\le\alpha_S R_S^\star$ and $R_F(B_F)\le\alpha_F R_F^\star$, then
$\RPO(B_S,B_F)\le \max\{\alpha_S,\alpha_F\}\Rstar.$

\end{corollary}


\subsection{The Competitive Analysis View} 
\label{sec:int-vs-fp}
This discussion motivates the following approach for improved quantization: 
\begin{enumerate}
    \item \textbf{Fix Prior}: identify a ``prior'' on the distribution of the groups we wish to quantize, e.g., split into spiky groups and flat groups. 
    \item \textbf{Find Grids}: Find two grids, each optimized for one of the group types.
    \item \textbf{Assignment}: When quantizing groups, assign each individual group to the ``better'' grid corresponding to its contents, in terms of quantization mean-square error (MSE).
\end{enumerate}

The natural target metric suggested by Corollary~\ref{thm:competitive} is to compare the risk (expected MSE) to the \textit{best single grid} that could be learned w.r.t. the given distribution. 

\textbf{Case study: INT4 and FP4.}
We instantiate the competitive analysis for the INT4/FP4 pair considered also by the IF4 concurrent work~\citep{cook2026adaptiveblockscaled}.
For weights, we use both a standardized Student-$t_\nu$ model with $\nu$ in the range matching real transformers (Table~\ref{tab:weight_fit}), as well as real samples from Qwen models. We examine block size $g = 16$. 

INT4 spaces its levels uniformly, while 
FP4 concentrates resolution near the origin.
Intuitively, this should make FP4 well-suited for spiky blocks, while INT4 is better for flat blocks.
To quantify this, we apply Corollary~\ref{thm:competitive}: split blocks into spiky ($\mu(X) \leq \tau$) and flat ($\mu(X) > \tau$), estimate the conditional factors $\alpha_S$ and $\alpha_F$ via Monte Carlo ($30{,}000$ training blocks, weighted Lloyd iterations for the conditional static optimum, $60{,}000$ validation blocks), and sweep $\tau$ to minimize $\beta = \max\{\alpha_S, \alpha_F\}$.
The resulting competitive factors are $\beta \approx 1.10$ for Normally-distributed data and $\beta \approx 1.13$ for Student-$t_5$ data at $g = 16$. Thus, we find that IF4 is  \textit{inferior} to the best possible static grid by 10--15\%, across distributions and parametrizations.
Thus, this bound reveals a structural limitation: both INT4 and FP4 are ``too suboptimal'' for their respective block classes.


\section{Proposed PO2 Grid Families}
\label{sec:methods}

We now focus on designing concrete grid pairs, aiming for pairs that are \textit{better than statically optimal}.

\paragraph{Training protocol.}
All learned grids are trained on absmax-normalized groups with block size $g=16$.
For a tensor block $x$, we normalize by $M=\max_i|x_i|$, evaluate MSE after de-normalization, and assign each block to the grid achieving lower MSE.
Unless otherwise stated, training uses pooled weight and activation groups from Qwen3-8B, Qwen3.5-9B, Mistral-7B-v0.3, and Phi-4 (2M sampled groups, balanced 1:1 between weights and activations).

The algorithm is summarized in Figure~\ref{fig:algorithm}: given a primary grid $B_1$ (fixed or learned), we collect the blocks on which $B_1$ incurs high error, initialize a secondary grid $B_2$ on this residual, and then alternate block assignment with Lloyd updates on both partitions. The grids are listed in Appendix~\ref{app:grids}. 

\paragraph{Snapping to FP8.} To maintain hardware feasibility, after training the grids, each point is rounded (snapped) to the nearest E4M3 value. We did not find this to yield significant error, see Appendix \ref{sec:format-snapping}.

\begin{figure}[t]
\begin{center}
\fbox{\parbox{0.92\textwidth}{
\textbf{The PO2 Algorithm for Residual Grid Learning}\\[4pt]
\textbf{Input:} Primary grid $B_1$ (fixed or learnable), pool of absmax-normalized blocks $\{x_1,\ldots,x_n\}$\\
\textbf{Output:} Grid pair $(B_1, B_2)$
\begin{enumerate}[leftmargin=1.5em,topsep=3pt,itemsep=1pt]
\item For each block $x_i$, compute $L_{B_1}(x_i)$.
\item Collect residual pool $\mathcal{R}=\{x_i : L_{B_1}(x_i)$ exceeds the median$\}$.
\item Initialize $B_2$ via weighted Lloyd iterations on $\mathcal{R}$.
\item \textbf{Repeat} until convergence:
  \begin{enumerate}[label=(\alph*),leftmargin=1.5em,itemsep=0pt]
  \item Assign each block to $\arg\min_{j\in\{1,2\}} L_{B_j}(x_i)$.
  \item Run weighted Lloyd on the blocks assigned to each updatable grid.
  \end{enumerate}
\item Snap code points to target format (e.g., FP8 E4M3).
\end{enumerate}
}}
\end{center}
\caption{Generic residual grid learning for PO2. When $B_1$ is fixed (NF4, Split87), only $B_2$ is updated in step~4(b). For MPO2, both grids are initialized and updated.}
\label{fig:algorithm}
\end{figure}

\subsection{PO2(NF4): A Grid Pair with a Normal and a Residual}

The QLoRA NF4 grid~\citep{dettmers2023qlora} is designed for Gaussian-distributed weights, using a quantile-based construction that places code points at equal-probability intervals under $\mathcal{N}(0,1)$.

Our PO2(NF4) proposal fixes the primary grid to a version of NF4 that is rounded (snapped) to FP8 and learns only the residual grid via the protocol in Figure~\ref{fig:algorithm}.
Our analysis predicts that this should be particularly effective for Student-$t$-like distributions with moderate degrees of freedom (e.g., $\nu \in [7,10]$, see Table~\ref{tab:weight_fit}): NF4 is near-optimal on the spiky tail, while the learned residual grid should specialize for the flatter blocks where NF4 wastes resolution. 
The resulting grids are visualized in Figure~\ref{fig:po2-grids}, confirming this intuition. 

\begin{figure}
    \centering
    \includegraphics[width=1\linewidth]{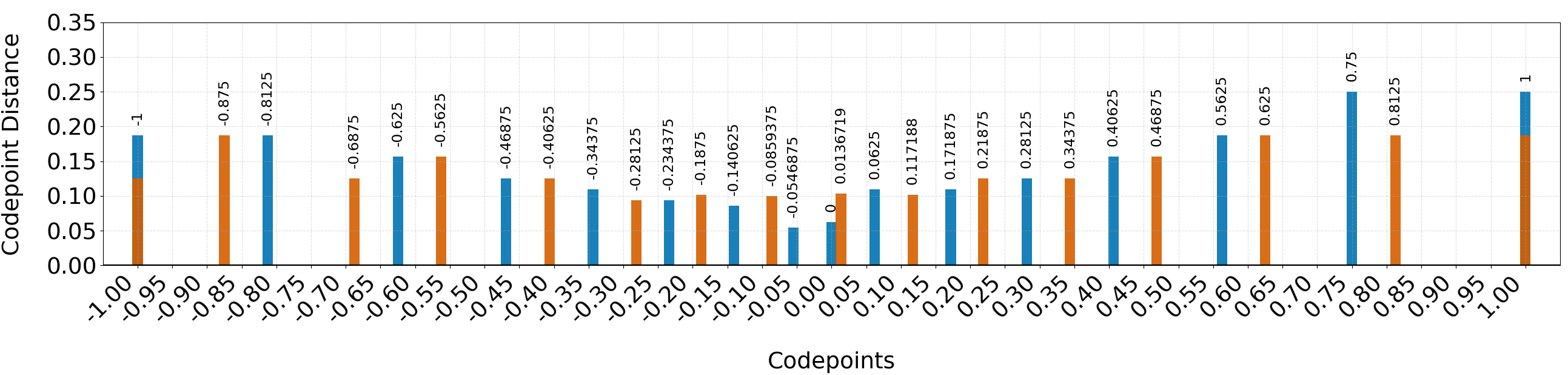}
    \caption{A 2D visualization of the PO2(NF4) grid pair. The ``height'' of each grid point representation is the distance to its right neighbor (except for the rightmost points, where the height represents the distance to the left). Exact codepoint values are annotated on top.}
    \label{fig:po2-grids}
\end{figure}

\subsection{PO2(Split87): Asymmetric MSE-Optimized Grid Plus Its Residual}

Is NF4 the ``optimal'' grid for a Gaussian prior?
We show this is not the case: we propose the Split87 grid, that starts from the same Gaussian prior but departs from QLoRA's quantile-based construction.
Whereas NF4 allocates code points to equalize probability mass under $\mathcal{N}(0,1)$, Split87 optimizes directly for MSE on the absmax-normalized distribution, which better reflects the quantization objective.
Crucially, Split87 encodes asymmetry with an explicit zero: eight negative levels, one zero, and seven positive nonzero levels.
The explicit zero eliminates dead-zone error for small values, which is important specifically for activations. We found it surprising that the 8+0+7 split consistently outperformed the 7+0+8 split in our experiments. 

The primary Split87 grid is learned by coordinate descent on the blockwise MSE objective over the training pool of tensors, then snapped to FP8.
The secondary grid is trained on the residual population following the same protocol as PO2(NF4).
Table~\ref{tab:wa-aggregate} shows that PO2(Split87) gives the lowest KL-Divergence and best average benchmarks recovery in the weight+activation setting.

\subsection{MPO2: Fully Learned Grid Pair}

Our next proposal is MPO2, which is a ``Meta'' grid pair that removes all distributional priors and learns both grids over pooled normalized weights and activation groups from our pool of models.
We first train a primary 16-level grid on the full pool via Lloyd iterations, and then train a secondary grid on this residual population.
We then alternate assignment and Lloyd updates for both grids until convergence.
The resulting FP8-rounded pair is listed in Appendix~\ref{app:grids}.

\subsection{SFP4: TensorCore-Compatible Shifted NVFP4}
\label{sec:gridflip}

The learned grids above demonstrate the statistical potential of dual-grid quantization, but current NVFP4 tensor cores decode only E2M1 values.
SFP4 asks how much of the dual-grid advantage can be recovered while staying hardware-compatible with current TensorCores.

The core idea is to offer, alongside the standard E2M1 grid centered at zero, two \textit{shifted} grids centered at $\pm c \cdot s$ from zero, where $c$ is a small per-tensor constant and $s$ is the block scale.
Blocks whose values cluster slightly below or above zero can then be quantized with denser coverage around their ``center of mass.''
This is compatible with the Four Over Six scale optimization~\citep{cook2025fouroversix}, which adjusts scale \textit{magnitude}; SFP4 adjusts the grid \textit{center}.

\paragraph{Three shifted NVFP4 grids.}
For each weight block with scale $s$ and E2M1 codes $c_k$, the encoder selects among three grids.
Grid~$\mathcal{A}$ decodes as standard NVFP4: $\hat{w}_k = s \, g(c_k)$, where $g$ maps E2M1 codes to their floating-point values.
Grids~$\mathcal{B}^{+}$ and $\mathcal{B}^{-}$ shift the grid center away from zero:
\[
\mathcal{B}^{+}\!:\;\hat{w}_k = s\,(g(c_k) + c_{+}),
\qquad
\mathcal{B}^{-}\!:\;\hat{w}_k = s\,(g(c_k) - c_{-}).
\]
Calibration found that the ``optimal'' shift values across most layers are $c_{+} = c_{-} = 0.5$. The normalized grids are given in Appendix~\ref{app:grids}. 


\paragraph{Scale search.} SFP4 is compatible with Four-Over-Six-style scale optimization. For instance, we could test divisors $d \in \{4, 4.5, 5, 5.5, 6\}$ with $s_d = x_{\max}/d$, selecting the divisor minimizing block MSE. Since we only apply SFP4 on the weights, this calibration can be done offline in the case of PTQ. 
Unless otherwise stated, we \textit{do not} apply this scale optimization in our experiments below, in order to isolate grid effects from orthogonal scale optimization effects; experimentally, we have found scale search to improve the MSE of all methods by between 5 and 10\%.


\paragraph{Grid flag.}
The three-way grid selector requires 2 bits per block.
We store both bits in the scale byte: bits~7:6 encode the grid selector ($0 = \mathcal{A}$, $1 = \mathcal{B}^{+}$, $2 = \mathcal{B}^{-}$), and bits~5:0 store the scale magnitude in E3M3 format (reduced from UE4M3). We found experimentally that the loss of accuracy due to the scale format reduction is negligible. 
Thus, the 2-bit flag adds no sideband storage.

\paragraph{Matmul decomposition.}
Both $\mathcal{B}$-grids decode as NVFP4 plus a per-block constant, so a standard matmul $y = Wx$ decomposes as
\[
\abovedisplayskip=6pt
\belowdisplayskip=6pt
y[m,n]
= \underbrace{\textrm{NVFP4\_MMA}(\tilde{W}, X)}_{y_{\mathrm{main}}}
\;-\; c\!\sum_b \sigma[m,b]\cdot s_W[m,b]\cdot X_{\mathrm{sum}}[b,n],
\]
where $\tilde{W}$ are base E2M1 weights, $\sigma[m,b] \in \{0,+1,-1\}$ encodes the grid choice, $s_W$ is the block scale, and $X_{\mathrm{sum}}[b,n] = \sum_{k \in b} x[k,n]$ is the activation block sum.

\section{Experimental Validation}

\subsection{Synthetic Grid Evaluations}

\begin{table}[t]
\centering
\caption{Grid comparison at $g=16$. Top: Monte Carlo MSE ($\times 10^3$, 2M samples) on standardized distributions. Bottom: WikiText-2 KL divergence ($\times 10^{-2}$) on real weight and activation-quantized models. Methods left of the vertical bar are executable on current hardware; methods right of it require per-block grid selection or non-E2M1 grid values. $G_{\mathrm{opt}}$ is the MSE-optimal grid.}
\label{tab:mse_unified}
\resizebox{\textwidth}{!}{
\begin{tabular}{l|ccc|cccccc|c}
\toprule
Dist./Model & INT4 & FP4 & SFP4 & BOF4 & NF4 & IF4 & PO2(NF4) & PO2(Split87) & MPO2 & $G_{\mathrm{opt}}$ \\
\midrule

Student-$t_5$     & 17.6 & 13.8 & \textbf{11.3} & 10.4 & 11.0 & 11.2 & 9.1  & 9.4  & \textbf{8.8} & 10.7 \\
Student-$t_7$     & 13.3 & 11.8  & \textbf{9.6}  & 8.3  & 9.2  & 9.3  & 7.5  & 7.7  & \textbf{7.1} & 8.5 \\
Student-$t_{10}$  & 11.0 & 10.7  & \textbf{8.6}  & 7.1  & 8.1  & 8.1  & 6.5  & 6.7  & \textbf{6.1} & 7.3 \\
Normal    &  7.6 &  8.9  & \textbf{7.0}  & 5.3  & 6.6  & 6.2  & 5.1  & 5.2  & \textbf{4.6} & 5.4 \\
\midrule
Wiki2 KL & 14.29 & 11.43 & \textbf{10.19} & 9.69 & 9.88 & 9.92 & 8.93 & \textbf{8.13} & 12.46 & --- \\
\bottomrule
\end{tabular}}
\end{table}

Table~\ref{tab:mse_unified} provides a detailed comparison of different grids first in terms of MSE on standard distributions (Normal and Student-t), and in terms of output KL-divergence averaged over real models (see Table~\ref{tab:wa-aggregate}). We separate currently-implementable constructions (left panel, INT4 to SFP4) from the general grids that require additional hardware support (right panel, BOF4 to $G_{\mathrm{opt}}$). 

We first observe that SFP4 is the best E2M1-native method, reducing MSE by 11-21\% relative to FP4, and being only 0--12\% worse relative to the two-grid IF4, while staying implementable.  
Experimentally, a single-shift version of SFP4 (2 grids instead of 3, not included) is considerably worse. 
Moving to general grids, observe that IF4 improves over the INT4/FP4 single-grid baselines, and is close to NF4, but loses to the optimal grid $G_{\mathrm{opt}}$ (confirming our analysis), as well as to the Normal prior grid pairs such as 
PO2(NF4) and PO2(Split87). The latter outperforms both the best possible single grid ($G_{\mathrm{opt}}$) and IF4, reducing MSE by 5--15\% and 15--20\% respectively. From KL, the rankings transfer well between the distributional setting and real models, except for MPO2, which seems to overfit model-specific distributions.
This confirms that non-E2M1 grids add substantial representational capacity and that the Normal prior is valid and generalizes better than model-specific fits.\looseness=-1

\subsection{Post-Training Accuracy}
\label{sec:experiments}

\paragraph{Setup.}
We quantize the Llama-3.2-3B-Instruct~\cite{grattafiori2024llama}, Qwen3-\{8B,14B\}~\cite{yang2025qwen3}, Qwen3.5-\{2B, 4B, 9B, 27B\}~\cite{qwen35blog} models, using RTN with absmax scaling and block size $g=16$. Scales are quantized to 8 bits (E4M3). In the majority of experiments,  we do not use Hadamard transforms, in order  to isolate the effect of grid choice (we provide results with Hadamard transforms in the Appendix \ref{app:ptq_results}).
We measure KL divergence against BF16 logits on WikiText-2 and C4, as well as Expected Acceptance Rate (EAR) between the original model and the quantized one~\cite{helcig2026statisticallylosslessquantizationlargelanguage}. We run models on downstream tasks using Harness~\cite{eval-harness} and report accuracies on  Winogrande~\cite{sakaguchi2021winogrande}, ARC-C, ARC-E~\cite{allenai:arc}, Lambada (standard)~\cite{paperno_2016_2630551}, PIQA~\cite{Bisk2020}, Hellaswag (10-shot)~\cite{zellers2019hellaswag}, MMLU~\cite{hendryckstest2021}, IFEval (Prompt) ~\cite{zhou2023instructionfollowingevaluationlargelanguage}, and GSM8K-CoT~\cite{cobbe2021gsm8k}. We compare several single-grids NVFP4, BOF4 \cite{blumenberg2025bof4}, NF4 \cite{dettmers2023qlora}, Split87, and several multi-grid variants IF4 (per-block INT4/FP4 selection~\citep{cook2026adaptiveblockscaled}), PO2(NF4), and PO2(Split87). We also compare with Four-Over-Six~\cite{cook2025fouroversix} and the SFP4 described in Section~\ref{sec:gridflip}. 

\paragraph{Weight-and-Activation PTQ Results.}
Tables~\ref{tab:wa-accuracy-permodel} and~\ref{tab:wa-aggregate} report the W4A4 results. Dual-grid methods consistently reduce KL divergence relative to single-grid NVFP4. Among single-grid variants, NF4 and Split87 snapped to FP8 perform well. PO2(Split87) achieves the lowest KL divergence, highest EAR  for every model (Appendix~\ref{app:full_tables}), as well as the best average accuracy across models, although differences among the top methods on downstream tasks are small. Identity-vs-Hadamard ablations on Qwen3.5-4B are reported in Appendix~\ref{app:ptq_results} (Tables~\ref{tab:qwen354b-wo-transforms} and~\ref{tab:qwen354b-wa-transforms}); in general, Hadamard rotation provides a small improvement in the W4A4 setting, while on W4A16 the results are mixed. Additionally, Figure~\ref{fig:seed_run} reports the seed-to-seed variation for Qwen3.5-4B.

\looseness=-1


\begin{table*}[t]
\centering
\caption{Accuracy of seven models quantized with weight and activation quantization (W4A4). Each per-model cell reports the mean over Winogrande, ARC-C, ARC-E, Lambada (standard), PIQA, Hellaswag (10-shot), MMLU, IFEval (Prompt), and GSM8K-CoT. The \textbf{Avg.\ Recovery} column reports the per-model recovery (quantized score divided by the BF16 baseline) averaged across all seven models. Best non-BF16 result per column is in \textbf{bold}; higher is better.}
\label{tab:wa-accuracy-permodel}
\setlength{\tabcolsep}{4pt}
\resizebox{\textwidth}{!}{%
\begin{tabular}{l ccccccc c}
\toprule
Method & L3.2-3B & Q3-8B & Q3-14B & Q3.5-2B & Q3.5-4B & Q3.5-9B & Q3.5-27B & \textbf{Avg.\ Recovery} \\
\midrule
\textit{BF16 (baseline)} & \textit{67.92} & \textit{74.11} & \textit{77.39} & \textit{61.12} & \textit{73.26} & \textit{75.64} & \textit{80.26} & \textit{100.00\%} \\
\midrule
PO2(Split87) & \textbf{66.39} & 73.15 & 76.58 & 56.97 & 70.71 & \textbf{75.50} & 80.30 & \textbf{97.85\%} \\
PO2(NF4) & 66.10 & \textbf{73.41} & 76.40 & 57.25 & 70.68 & 75.14 & 79.95 & 97.74\% \\
NF4 & 65.74 & 72.91 & 76.45 & 56.69 & 71.11 & 75.07 & 80.35 & 97.59\% \\
Split87 & 65.68 & 72.89 & 76.65 & \textbf{57.45} & 70.48 & 74.76 & 80.16 & 97.57\% \\
IF4 & 65.31 & 72.74 & \textbf{76.80} & 56.71 & 70.76 & 75.04 & \textbf{80.39} & 97.47\% \\
SFP4 & 65.61 & 73.05 & 76.18 & 56.35 & \textbf{71.47} & 74.05 & 80.20 & 97.31\% \\
4over6 & 65.51 & 72.69 & 76.59 & 56.23 & 70.64 & 74.88 & 80.25 & 97.27\% \\
BOF4 & 65.25 & 73.33 & 76.50 & 55.78 & 70.67 & 75.26 & 80.13 & 97.27\% \\
NVFP4 & 65.04 & 72.77 & 75.92 & 56.63 & 69.72 & 74.70 & 80.37 & 96.97\% \\
NVINT4 & 64.00 & 71.82 & 75.89 & 53.20 & 70.31 & 73.86 & 79.99 & 95.65\% \\
\bottomrule
\end{tabular}%
}
\end{table*}

\begin{table*}[t]
\centering
\caption{Average of the benchmarks and KL divergences on weight and activation quantization (W4A4) on seven models, plus the EAR ( Expected Acceptance Rate). KL columns measure the divergence between the quantized and BF16 output distributions on the indicated calibration set ($\downarrow$, lower is better); EAR$_{\text{top10},\,C4}$ is higher-is-better. \emph{Avg} is the mean over Winogrande, ARC-C, ARC-E, Lambada (standard), PIQA, Hellaswag (10-shot), MMLU, IFEval (Prompt), and GSM8K-CoT. \emph{Avg.\ Recovery} is the per-model recovery on \emph{Avg} (quantized $\div$ BF16) averaged across the seven models. Each cell is the mean across 7 models; best non-BF16 result is in \textbf{bold}.}
\label{tab:wa-aggregate}
\setlength{\tabcolsep}{4pt}
\resizebox{\textwidth}{!}{%
\begin{tabular}{l cccc c c}
\toprule
 & \multicolumn{4}{c}{\textbf{Distribution-level}} & \textbf{Benchmarks} & \\
\cmidrule(lr){2-5}
Method & KL$_{\text{top10},\,C4}$ $\downarrow$ & EAR$_{\text{top10},\,C4}$ $\uparrow$ & KL$_{\text{Wiki2}}$ $\downarrow$ & KL$_{\text{C4}}$ $\downarrow$ & Avg $\uparrow$ & \textbf{Avg.\ Recovery} $\uparrow$ \\
\midrule
\textit{BF16 (baseline)} & \textit{0.0000} & \textit{1.0000} & \textit{0.0000} & \textit{0.0000} & \textit{72.81} & \textit{100.00\%} \\
\midrule
PO2(Split87) & \textbf{0.0364} & \textbf{0.9235} & \textbf{0.0813} & \textbf{0.0464} & \textbf{71.37} & \textbf{97.85\%} \\
PO2(NF4) & 0.0403 & 0.9196 & 0.0893 & 0.0514 & 71.27 & 97.74\% \\
NF4 & 0.0460 & 0.9141 & 0.0988 & 0.0589 & 71.19 & 97.59\% \\
Split87 & 0.0413 & 0.9186 & 0.0897 & 0.0524 & 71.15 & 97.57\% \\
IF4 & 0.0456 & 0.9142 & 0.0992 & 0.0585 & 71.11 & 97.47\% \\
SFP4 & 0.0488 & 0.9118 & 0.1019 & 0.0621 & 70.99 & 97.31\% \\
4over6 & 0.0491 & 0.9116 & 0.1023 & 0.0628 & 70.97 & 97.27\% \\
BOF4 & 0.0446 & 0.9155 & 0.0969 & 0.0568 & 70.99 & 97.27\% \\
NVFP4 & 0.0543 & 0.9067 & 0.1143 & 0.0696 & 70.74 & 96.97\% \\
NVINT4 & 0.0702 & 0.8925 & 0.1429 & 0.0919 & 69.87 & 95.65\% \\
\bottomrule
\end{tabular}%
}
\end{table*}



\subsection{Pretraining with Multiple Grids}
\label{sec:pretraining}

We evaluate multi-grid quantization for pretraining using the small-scale protocol from QuEST~\citep{panferov2025queststabletrainingllms}: Llama-like~\citep{touvron2023llamaopenefficientfoundation} transformer models with 30M, 50M and 100M parameters trained on C4 with quantized weights and activations via straight-through gradients.
Figure~\ref{fig:pretraining} reports the C4 validation loss gap relative to BF16 across several data/parameter ratios. Full hyper-parameters and weight-only QAT measurements are presented in Appendix~\ref{app:qat}.
The best single-grid method (BOF4) consistently outperforms some double-grid methods (IF4) but loses to the best double-grid method (PO2(Split87)).

\begin{figure}[t]
\centering
\includegraphics[width=\textwidth]{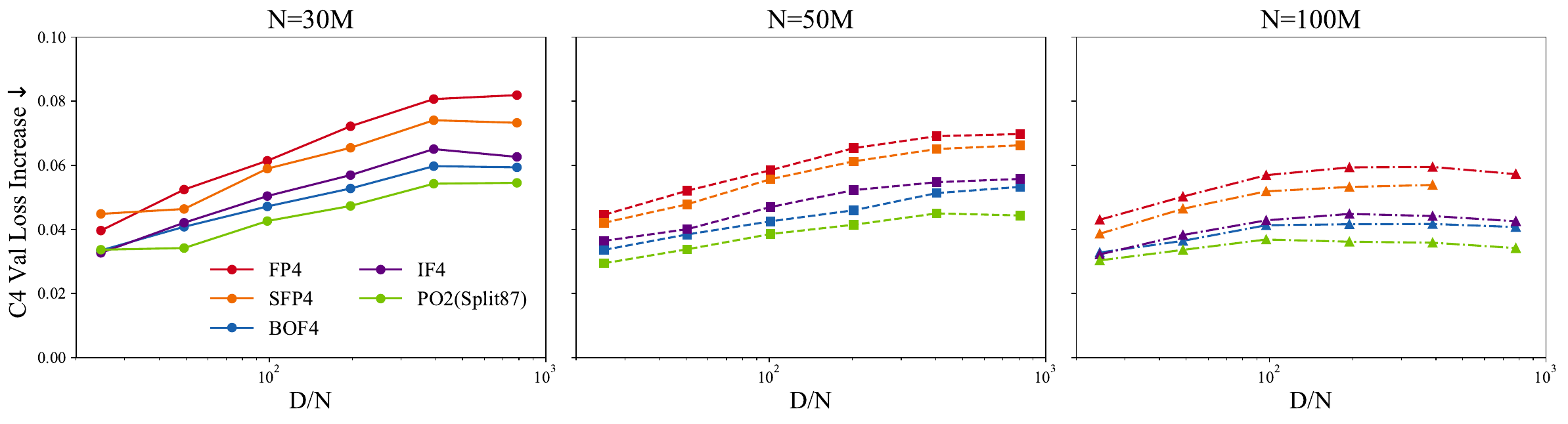}
\caption{Final C4 validation loss gap relative to BF16 for quantized pre-training of Llama-like models with $N$ parameters with $D/N$ tokens-per-parameter.}
\label{fig:pretraining}
\end{figure}

\subsection{SFP4: Quality and Kernel Speed}

\paragraph{Quantization quality.}
Table~\ref{tab:gridflip_quality} compares SFP4 variants against single-grid baselines in weight-only KL divergence.
The two-shift variant with Four Over Six scale search (SFP4-2c+4/6) improves Qwen3.5-0.8B KL by 23\% over NVFP4 and outperforms IF4 among E2M1-compatible methods.
The one-shift variant (SFP4-1c+4/6) improves over IF4 on Qwen3.5-0.8B, but is weaker on Qwen3-8B.
PO2(NF4) remains better across all models, but has no execution path on current hardware.\looseness=-1

  \begin{table}[t]\vspace{-1em}
  \centering
  \caption{Shifted NVFP4 (SFP4) weight-only WikiText-2 KL divergence (seq=4096, full eval); 1c uses one shift, 2c two.}
  \label{tab:gridflip_quality}
  \begin{tabular}{lccc}
  \toprule
  Method & Qwen3.5-0.8B & Qwen3.5-4B & Qwen3-8B \\
  \midrule
  NVFP4          & 0.0905 & 0.0697 & 0.0547 \\
  IF4            & 0.0771 & 0.0602 & 0.0470 \\
  Four Over Six  & 0.0811 & 0.0651 & 0.0483 \\
  SFP4-1c        & 0.0828 & 0.0681 & 0.0557 \\
  SFP4-2c        & 0.0823 & 0.0610 & 0.0547 \\
  SFP4-1c+4/6    & 0.0713 & 0.0655 & 0.0515 \\
  SFP4-2c+4/6    & 0.0701 & 0.0619 & 0.0527 \\
  PO2(NF4)       & 0.0604 & 0.0568 & 0.0383 \\
  \bottomrule
  \end{tabular}\vspace{-1em}
  \end{table}


\paragraph{Implementation and overhead.}
We provide an implementation using two GEMMs: the main GEMM is unmodified NVFP4, and a BF16 correction GEMM contracts a signed correction matrix $\sigma[m,b] \cdot s_W[m,b]$ (shape $M \times K/g$, stored in BF16) with activation block sums, as detailed in Appendix~\ref{sec:sfp4-kernels}
For sufficiently large shapes ($M=4096, N=5120,  K=\num{27648}$), the overhead of SFP4 can be limited to about 50\% slower than the NVFP4 path, see Table \ref{tab:sfp4-benchmark} in the appendix.

\section{Discussion and Conclusions}
As demonstrated above, the standard (NV)FP4 quantization grid is suboptimal in several ways. For one, it wastes two representations on zero, and other static grids of 16 levels, such as NF4, can achieve better results.
Second, an entire bit in the microscale is completely unused, and could carry additional information. This could be a sign in case of asymmetric base grids (BOF4), or an indicator for a shift; in its most general form, it just indicates the selection of one out of two possible quantization grids used for that particular group. Two grids are generally better than one, with differences most pronounced at small group sizes, which are not governed by the Central Limit Theorem. In particular, it appears that the ability to specialize grids based on whether they represent \textquote{spiky} or \textquote{flat} groups is generally beneficial. However, which construction to choose for these two grids is far less clear: While there is a significant improvement when going from standard FP4 to any of the two-grid representations proposed here, the differences between PO2 grids themselves are often close to noise level, and their ranking can change from model to model.
As such, it seems premature to ask hardware vendors to implement one specific format, such as IF4, directly in silicon. Instead, we believe it would be more beneficial if future hardware provides configurable lookup-tables directly in the datapath. As the number of entries is low ($2 \times 16$), and each entry in itself can be small, (we snapped all grids to E4M3), such lookup tables would not take a huge chip-space budget and they could be pipelined. And if the table entries are small enough,
hardware bottlenecks shift towards the accumulators \citep{blumenfeld2024towards}. Thus, we expect throughput could be maintained at moderate hardware cost.\footnote{Note that, e.g., on AMD MI355X, FP4 and FP6 have the same throughput}

Do we actually need hardware support? In the memory-bound regime, i.e., in low batch-size decoding, table lookup can be implemented in software~\citep{guo2024fast} and arithmetic is handled in higher precision at negligible cost. We provide such a kernel for the PO2-setup in Appendix~\ref{sec:po2-kernels}. In the compute-bound case, that is, during prefill or training, speed is dictated by the raw arithmetic throughput of tensor units at the decoded precision. Even if software decoding could be perfectly overlapped with tensor compute, the only gains would be in global memory bandwidth and cache consumption; transfer from registers into tensor cores use the wide representation.
With SFP4, we provide a compromise, by having highly structured grids that allow the main computation to be performed in 4 bit, while the cost of the correction kernel is low enough that the total compute budget is still lower than doing FP8 arithmetic.\looseness=-1

Finally, one might ask whether, if we request enhanced hardware support, the ask shouldn't be for even stronger formats such as QTIP\citep{tseng2024qtip}. However, such formats excel in the weight-only setting, where expensive compression algorithms can be run offline, but do not support efficient implementations once activations need to be quantized on the fly. In contrast, the quantization procedure for PO2 is straightforward, making it applicable to weight+activation quantization as well.

All in all, we have shown the power of multiple grids for 4-bit quantization. While we provide several successful grid pairs, we believe there is a huge space for further optimization, especially if one starts to pick grid-pairs optimized for individual tensors.
The fixed-function graphics pipeline gave way to programmable shaders in the 2000s, unlocking compute uses that chip designers could never have anticipated. Perhaps it is time for a similar evolution in low-bit matrix multiplication.

\paragraph{Acknowledgments.} 
This research was funded in part by the Austrian Science Fund (FWF) 10.55776/COE12, i.e., the
Bilateral AI Cluster of Excellence, and through generous research support by NVIDIA and Google. This work
was supported under project ID 40 as part of the Swiss AI Initiative, through a grant from the ETH
Domain and computational resources provided by the Swiss National Supercomputing Centre (CSCS)
under the Alps infrastructure. 
We would like to thank Tijmen
Blankevoort (NVIDIA), for useful discussions. Additionally, we would like
to thank Datacrunch/Verda, in particular Paul Chang, for hardware support that was essential to this project. We would also like
to thank Roberto L. Castro for help in reviewing our kernels. 

\bibliographystyle{plainnat}
\bibliography{references}

@article{nagel2021white,
  title={A White Paper on Neural Network Quantization},
  author={Nagel, Markus and Amjad, Rana Ali and van Baalen, Mart and Louizos, Christos and Blankevoort, Tijmen},
  journal={arXiv preprint arXiv:2106.08295},
  year={2021}
}

@inproceedings{frantar2023gptq,
  title={{GPTQ}: Accurate Post-Training Quantization for Generative Pre-Trained Transformers},
  author={Frantar, Elias and Ashkboos, Saleh and Hoefler, Torsten and Alistarh, Dan},
  booktitle={International Conference on Learning Representations},
  year={2023}
}

@article{sakaguchi2021winogrande,
  author = {Sakaguchi, Keisuke and Bras, Ronan Le and Bhagavatula, Chandra and Choi, Yejin},
  journal = {Communications of the ACM},
  number = {9},
  pages = {99--106},
  publisher = {ACM New York, NY, USA},
  title = {Winogrande: An adversarial winograd schema challenge at scale},
  volume = {64},
  year = {2021},
}

@dataset{paperno_2016_2630551,
  author       = {Paperno, Denis and
                  Kruszewski, Germán and
                  Lazaridou, Angeliki and
                  Pham, Quan Ngoc and
                  Bernardi, Raffaella and
                  Pezzelle, Sandro and
                  Baroni, Marco and
                  Boleda, Gemma and
                  Fernández, Raquel},
  title        = {The LAMBADA dataset},
  month        = aug,
  year         = 2016,
  publisher    = {Zenodo},
  doi          = {10.5281/zenodo.2630551},
  url          = {https://doi.org/10.5281/zenodo.2630551},
}

@inproceedings{Bisk2020,
  author = {Yonatan Bisk and Rowan Zellers and
            Ronan Le Bras and Jianfeng Gao
            and Yejin Choi},
  title = {PIQA: Reasoning about Physical Commonsense in
           Natural Language},
  booktitle = {Thirty-Fourth AAAI Conference on
               Artificial Intelligence},
  year = {2020},
}

@inproceedings{zellers2019hellaswag,
  title={Hellaswag: Can a machine really finish your sentence?},
  author={Zellers, Rowan and Holtzman, Ari and Bisk, Yonatan and Farhadi, Ali and Choi, Yejin},
  booktitle={Proceedings of the 57th annual meeting of the association for computational linguistics},
  pages={4791--4800},
  year={2019}
}

@article{grattafiori2024llama,
  title={The llama 3 herd of models},
  author={Grattafiori, Aaron and Dubey, Abhimanyu and Jauhri, Abhinav and Pandey, Abhinav and Kadian, Abhishek and Al-Dahle, Ahmad and Letman, Aiesha and Mathur, Akhil and Schelten, Alan and Vaughan, Alex and others},
  journal={arXiv preprint arXiv:2407.21783},
  year={2024}
}

@article{yang2025qwen3,
  title={Qwen3 technical report},
  author={Yang, An and Li, Anfeng and Yang, Baosong and Zhang, Beichen and Hui, Binyuan and Zheng, Bo and Yu, Bowen and Gao, Chang and Huang, Chengen and Lv, Chenxu and others},
  journal={arXiv preprint arXiv:2505.09388},
  year={2025}
}

@misc{qwen35blog,
    title = {Qwen3.5: Accelerating Productivity with Native Multimodal Agents},
    url = {https://qwen.ai/blog?id=qwen3.5},
    author = {Qwen Team},
    month = {February},
    year = {2026}
}

@article{cobbe2021gsm8k,
  title={Training Verifiers to Solve Math Word Problems},
  author={Cobbe, Karl and Kosaraju, Vineet and Bavarian, Mohammad and Chen, Mark and Jun, Heewoo and Kaiser, Lukasz and Plappert, Matthias and Tworek, Jerry and Hilton, Jacob and Nakano, Reiichiro and Hesse, Christopher and Schulman, John},
  journal={arXiv preprint arXiv:2110.14168},
  year={2021}
}

@misc{zhou2023instructionfollowingevaluationlargelanguage,
      title={Instruction-Following Evaluation for Large Language Models}, 
      author={Jeffrey Zhou and Tianjian Lu and Swaroop Mishra and Siddhartha Brahma and Sujoy Basu and Yi Luan and Denny Zhou and Le Hou},
      year={2023},
      eprint={2311.07911},
      archivePrefix={arXiv},
      primaryClass={cs.CL},
      url={https://arxiv.org/abs/2311.07911}, 
}

@article{hendryckstest2021,
  title={Measuring Massive Multitask Language Understanding},
  author={Dan Hendrycks and Collin Burns and Steven Basart and Andy Zou and Mantas Mazeika and Dawn Song and Jacob Steinhardt},
  journal={Proceedings of the International Conference on Learning Representations (ICLR)},
  year={2021}
}

@article{allenai:arc,
      author    = {Peter Clark  and Isaac Cowhey and Oren Etzioni and Tushar Khot and
                    Ashish Sabharwal and Carissa Schoenick and Oyvind Tafjord},
      title     = {Think you have Solved Question Answering? Try ARC, the AI2 Reasoning Challenge},
      journal   = {arXiv:1803.05457v1},
      year      = {2018},
}

@misc{eval-harness,
  author       = {Gao, Leo and Tow, Jonathan and Abbasi, Baber and Biderman, Stella and Black, Sid and DiPofi, Anthony and Foster, Charles and Golding, Laurence and Hsu, Jeffrey and Le Noac'h, Alain and Li, Haonan and McDonell, Kyle and Muennighoff, Niklas and Ociepa, Chris and Phang, Jason and Reynolds, Laria and Schoelkopf, Hailey and Skowron, Aviya and Sutawika, Lintang and Tang, Eric and Thite, Anish and Wang, Ben and Wang, Kevin and Zou, Andy},
  title        = {The Language Model Evaluation Harness},
  month        = 07,
  year         = 2024,
  publisher    = {Zenodo},
  version      = {v0.4.3},
  doi          = {10.5281/zenodo.12608602},
  url          = {https://zenodo.org/records/12608602}
}

@misc{helcig2026statisticallylosslessquantizationlargelanguage,
      title={Statistically-Lossless Quantization of Large Language Models}, 
      author={Michael Helcig and Eldar Kurtic and Dan Alistarh},
      year={2026},
      eprint={2605.02404},
      archivePrefix={arXiv},
      primaryClass={cs.LG},
      url={https://arxiv.org/abs/2605.02404}, 
}

@inproceedings{lin2024awq,
  title={{AWQ}: Activation-Aware Weight Quantization for {LLM} Compression and Acceleration},
  author={Lin, Ji and Tang, Jiaming and Tang, Haotian and Yang, Shang and Dang, Xingyu and Gan, Chuang and Han, Song},
  booktitle={Proceedings of Machine Learning and Systems},
  year={2024}
}

@misc{blumenberg2025bof4,
  title         = {{Improving Block-Wise LLM Quantization by 4-bit Block-Wise Optimal Float (BOF4): Analysis and Variations}},
  author        = {Blumenberg, Patrick and Graave, Thomas and Fingscheidt, Tim},
  year          = {2025},
  month         = may,
  eprint        = {2505.06653},
  archivePrefix = {arXiv},
  primaryClass  = {cs.LG},
  doi           = {10.48550/arXiv.2505.06653},
  url           = {https://arxiv.org/abs/2505.06653}
}

@inproceedings{dettmers2023qlora,
  title={{QLoRA}: Efficient Finetuning of Quantized {LLMs}},
  author={Dettmers, Tim and Pagnoni, Artidoro and Holtzman, Ari and Zettlemoyer, Luke},
  booktitle={Advances in Neural Information Processing Systems},
  year={2023}
}

@article{rouhani2023microscaling,
  title={Microscaling Data Formats for Deep Learning},
  author={Rouhani, Bita Darvish and Zhao, Ritchie and More, Ankit and Hall, Mathew and Khodamoradi, Alireza and Deng, Summer and Choudhary, Dhruv and Cornea, Marius and Dellinger, Eric and Denolf, Kristof and Dusan, Stosic and Elango, Venmugil and Golub, Maximilian and Heinecke, Alexander and James-Roxby, Phil and Jani, Dharmesh and Kolhe, Gaurav and Langhammer, Martin and Li, Ada and Melnick, Levi and Mesmakhosroshahi, Maral and Rodriguez, Andres and Schulte, Michael and Shafipour, Rasoul and Shao, Lei and Siu, Michael and Dubey, Pradeep and Micikevicius, Paulius and Naumov, Maxim and Verrilli, Colin and Wittig, Ralph and Burger, Doug and Chung, Eric},
  journal={arXiv preprint arXiv:2310.10537},
  year={2023}
}

@article{khailany2025fp4,
  title={4-bit Floating Point Quantization for Inference},
  author={Khailany, B. and others},
  journal={NVIDIA Technical Report},
  year={2025}
}

@misc{chmiel2025fp4wayfullyquantized,
      title={FP4 All the Way: Fully Quantized Training of LLMs}, 
      author={Brian Chmiel and Maxim Fishman and Ron Banner and Daniel Soudry},
      year={2025},
      eprint={2505.19115},
      archivePrefix={arXiv},
      primaryClass={cs.LG},
      url={https://arxiv.org/abs/2505.19115}, 
}

@article{castro2025quartet,
  title={{Quartet}: Native {FP4} Training Can Be Optimal for Large Language Models},
  author={Castro, Roberto L. and Panferov, Andrei and Tabesh, Soroush and Sieberling, Oliver and Chen, Jiale and Nikdan, Mahdi and Ashkboos, Saleh and Alistarh, Dan},
  journal={arXiv preprint arXiv:2505.14669},
  year={2025}
}

@article{panferov2026quartetii,
  title={{Quartet II}: Accurate {LLM} Pre-Training in {NVFP4} by Improved Unbiased Gradient Estimation},
  author={Panferov, Andrei and Schultheis, Erik and Tabesh, Soroush and Alistarh, Dan},
  journal={arXiv preprint arXiv:2601.22813},
  year={2026}
}

@article{cook2026adaptiveblockscaled,
  title={Adaptive Block-Scaled Data Types},
  author={Cook, Jack and Lee, Hyemin S. and Le, Kathryn and Guo, Junxian and Traverso, Giovanni and Chandrakasan, Anantha P. and Han, Song},
  journal={arXiv preprint arXiv:2603.28765},
  year={2026}
}

@article{cook2025fouroversix,
  title={Four Over Six: More Accurate {NVFP4} Quantization with Adaptive Block Scaling},
  author={Cook, Jack and Guo, Junxian and Xiao, Guangxuan and Lin, Yujun and Han, Song},
  journal={arXiv preprint arXiv:2512.02010},
  year={2025}
}

@article{lloyd1982least,
  title={Least Squares Quantization in {PCM}},
  author={Lloyd, Stuart P.},
  journal={IEEE Transactions on Information Theory},
  volume={28},
  number={2},
  pages={129--137},
  year={1982}
}

@article{max1960quantizing,
  title={Quantizing for Minimum Distortion},
  author={Max, Joel},
  journal={IRE Transactions on Information Theory},
  volume={6},
  number={1},
  pages={7--12},
  year={1960}
}

@article{
elangovan2025lobcq,
title={{LO}-{BCQ}: Locally Optimal Block Clustered Quantization for 4-bit (W4A4) {LLM} Inference},
author={Reena Elangovan and Charbel Sakr and Anand Raghunathan and Brucek Khailany},
journal={Transactions on Machine Learning Research},
issn={2835-8856},
year={2025},
url={https://openreview.net/forum?id=loWISTqGwW},
note={Featured Certification, J2C Certification}
}

@article{martin2021implicit,
  title={Implicit Self-Regularization in Deep Neural Networks: Evidence from Random Matrix Theory and Implications for Training},
  author={Martin, Charles H. and Mahoney, Michael W.},
  journal={Journal of Machine Learning Research},
  volume={22},
  number={165},
  pages={1--73},
  year={2021}
}

@article{hodgkinson2021multiplicative,
  title={Multiplicative Noise and Heavy Tails in Stochastic Optimization},
  author={Hodgkinson, Liam and Mahoney, Michael W.},
  journal={International Conference on Machine Learning},
  year={2021}
}

@inproceedings{ashkboos2024quarot,
  title={{QuaRot}: Outlier-Free 4-Bit Inference in Rotated {LLMs}},
  author={Ashkboos, Saleh and Mohtashami, Amirkeivan and Croci, Maximilian L. and Li, Bo and Cameron, Pashmina and Jaggi, Martin and Alistarh, Dan and Hoefler, Torsten and Hensman, James},
  booktitle={Advances in Neural Information Processing Systems},
  year={2024}
}

@inproceedings{chee2024quip,
  title={{QuIP}\#: Even Better {LLM} Quantization with Hadamard Incoherence and Lattice Codebooks},
  author={Tseng, Albert and Chee, Jerry and Sun, Qingyao and Kuleshov, Volodymyr and De Sa, Christopher},
  booktitle={Proceedings of the 41st International Conference on Machine Learning},
  pages={48630--48656},
  year={2024}
}

@inproceedings{kim2024squeezellm,
  title={{SqueezeLLM}: Dense-and-Sparse Quantization},
  author={Kim, Sehoon and Hooper, Coleman and Gholami, Amir and Dong, Zhen and Li, Xiuyu and Shen, Sheng and Mahoney, Michael W. and Keutzer, Kurt},
  booktitle={International Conference on Machine Learning},
  year={2024}
}

@inproceedings{malinovskii2025higgs,
  title={{HIGGS}: Pushing the Limits of Large Language Model Quantization via the Linearity Theorem},
  author={Malinovskii, Vladimir and Panferov, Andrei and Ilin, Ivan and Guo, Han and Richt{\'a}rik, Peter and Alistarh, Dan},
  booktitle={Proceedings of the 2025 Conference of the Nations of the Americas Chapter of the Association for Computational Linguistics: Human Language Technologies},
  pages={10857--10886},
  year={2025}
}

@misc{touvron2023llamaopenefficientfoundation,
      title={LLaMA: Open and Efficient Foundation Language Models}, 
      author={Hugo Touvron and Thibaut Lavril and Gautier Izacard and Xavier Martinet and Marie-Anne Lachaux and Timothée Lacroix and Baptiste Rozière and Naman Goyal and Eric Hambro and Faisal Azhar and Aurelien Rodriguez and Armand Joulin and Edouard Grave and Guillaume Lample},
      year={2023},
      eprint={2302.13971},
      archivePrefix={arXiv},
      primaryClass={cs.CL},
      url={https://arxiv.org/abs/2302.13971}, 
}

@misc{panferov2025queststabletrainingllms,
      title={QuEST: Stable Training of LLMs with 1-Bit Weights and Activations}, 
      author={Andrei Panferov and Jiale Chen and Soroush Tabesh and Roberto L. Castro and Mahdi Nikdan and Dan Alistarh},
      year={2025},
      eprint={2502.05003},
      archivePrefix={arXiv},
      primaryClass={cs.LG},
      url={https://arxiv.org/abs/2502.05003}, 
}

@inproceedings{blumenfeld2024towards,
title={Towards Cheaper Inference in Deep Networks with Lower Bit-Width Accumulators},
author={Yaniv Blumenfeld and Itay Hubara and Daniel Soudry},
booktitle={The Twelfth International Conference on Learning Representations},
year={2024},
url={https://openreview.net/forum?id=oOwDQl8haC}
}

@inproceedings{guo2024fast,
  title={Fast matrix multiplications for lookup table-quantized llms},
  author={Guo, Han and Brandon, William and Cholakov, Radostin and Ragan-Kelley, Jonathan and Xing, Eric and Kim, Yoon},
  booktitle={Findings of the Association for Computational Linguistics: EMNLP 2024},
  pages={12419--12433},
  year={2024}
}

@article{tseng2024qtip,
  title={Qtip: Quantization with trellises and incoherence processing},
  author={Tseng, Albert and Sun, Qingyao and Hou, David and De, Christopher},
  journal={Advances in Neural Information Processing Systems},
  volume={37},
  pages={59597--59620},
  year={2024}
}

@article{tseng2025training,
  title={Training llms with mxfp4},
  author={Tseng, Albert and Yu, Tao and Park, Youngsuk},
  journal={arXiv preprint arXiv:2502.20586},
  year={2025}
}

\appendix

\section{Limitations and Broader Impacts}
\label{sec:limitations}

The theory uses i.i.d.\ block models and Student-$t$ fits, while real transformer tensors have structure from channels, heads, experts, and training dynamics.
The empirical agreement suggests the model is useful, but it should not be read as a full generative model of weights or activations.
The post-training experiments cover several model families but still emphasize Qwen-family evaluations, and the pretraining runs are small compared with frontier LLM training.
SFP4 is hardware-faithful but not yet a fully fused production kernel; the current wrapper path has large overhead, and the cleanest one-bit variant is weaker than non-E2M1 learned grids.

The broader impact is primarily efficiency.
Better 4-bit quantization can reduce the cost and energy of LLM training and serving, which may broaden access to high-quality models.
It may also reduce the cost of harmful deployments of generative models.
This work does not release a new pretrained model or dataset; the main artifact is a quantization method and empirical analysis.

\section{Complete Proofs}
\label{app:proofs}

\subsection{Problem Setup}

We now re-iterate the complete argument. Let $g$ be the quantization group size.  A block is denoted
$X=(X_1,\ldots,X_g)\in\R^g$.  Unless stated otherwise, the coordinates in a
block are i.i.d. from a distribution with finite second moment; the
bounded-support assumption needed for Theorem~\ref{thm:no-asymptotic-advantage}
is stated separately below.  We write
\[
        Z_i := |X_i|,\qquad
        M := \max_{1\le i\le g} Z_i .
\]
When $M>0$, define the normalized magnitudes
\[
        A_i := Z_i/M \in [0,1].
\]
When $M=0$, the block is identically zero; in that case, we use the
zero-block convention $A_i=0$ and quantization loss equal to zero. (This
 avoids imposing the assumption $\mathbb{P}[X_i=0]=0$.)

For notational simplicity, we first focus on the following parameterization of the positive half of an eight-point signed codebook by
\[
        B=(b_0,\ldots,b_7),\qquad 0=b_0\le b_1\le\cdots\le b_6\le b_7=1,
\]
and denote the compact grid space by
\[
        \Cbar := \{B:0=b_0\le b_1\le\cdots\le b_6\le b_7=1\}.
\]
The strict class used by actual eight-level grids is
\[
        \Cstrict := \{B:0=b_0<b_1<\cdots<b_6<b_7=1\}.
\]
For the argument below, one may use either $\Cbar$ or $\Cstrict$:
$\Cstrict$ is dense in $\Cbar$, and the risks are continuous in the sorted grid
coordinates. 
For $B\in\Cbar$, let $q_B(a)$ be a closest element of $B$ to $a\in[0,1]$;
ties are broken randomly.  For $M>0$, define the signed
absmax quantizer
\[
        Q_B(x;M) := \operatorname{sgn}(x)\,M\,q_B(|x|/M).
\]
On the zero block $M=0$, we set by definition $Q_B(0;0)=0$.  Define
\[
        \psi_B(a) := \min_{0\le j\le 7} (a-b_j)^2,
        \qquad a\in[0,1].
\]
The block loss is
\[
        L_B^{(g)}(X)
        :=\frac1g\sum_{i=1}^g\bigl(X_i-Q_B(X_i;M)\bigr)^2.
\]
Equivalently, with the zero-block convention,
\[
        L_B^{(g)}(X)
        =M^2\frac1g\sum_{i=1}^g \psi_B(A_i).            \tag{1}\label{eq:scalar-reduction}
\]
Thus, for i.i.d. coordinates in a block, exchangeability gives
\[
        R_g(B) := \E L_B^{(g)}(X)
        = \E\!\left[M^2\psi_B(A_1)\right].
\]
The best single-grid risk and the adaptive two-grid risk are
\[
        R_g^\star := \inf_{B\in\Cbar} R_g(B),
        \qquad
        R_{\mathrm{PO2},g}^\star
        := \inf_{B_1,B_2\in\Cbar}
           \E\!\left[\min\{L_{B_1}^{(g)}(X),L_{B_2}^{(g)}(X)\}\right].
\]
\subsection{Theoretical Analysis}

The general multi-grid objective is the following.  Given blocks
$X^{(1)},\ldots,X^{(n)}\in\R^g$, a codebook class $\mathcal{A}$, and
$k\ge1$ grids $B_1,\ldots,B_k\in\mathcal{A}$, choose the grids to minimize
\[
        \frac1n\sum_{r=1}^n \min_{1\le j\le k}
        L_{B_j}^{(g)}(X^{(r)}).
\]
For the two-grid setting studied here, $k=2$ and
$\mathcal{A}$ is either $\Cbar$ or the dense strict subclass $\Cstrict$.

The next result formalizes the fact that absmax-normalized blocks have no
asymptotic benefit from per-block selection among two fixed grids once the block
\newtheorem*{theorem*}{Theorem}
\begin{theorem*}[No asymptotic advantage of two grids]\label{thm:no-asymptotic-advantage}
Let $X_1,X_2,\ldots$ be i.i.d. real-valued random variables, and set
$Z_i=|X_i|$.  Assume $\mathbb{P}(Z_1\le x_{\max})=1$ and, for every $\eps>0$, $\mathbb{P}(Z_1>x_{\max}-\eps)>0.$
Then
\[
        \lim_{g\to\infty}\bigl(R_g^\star-R_{\mathrm{PO2},g}^\star\bigr)=0.
\]

\end{theorem*}



\begin{proof}
Let $Z_i=|X_i|$ and $M_g=\max_{i\le g}Z_i$.  The definition ensures that $x_{\max}$ is the essential supremum, formally 
$x_{\max}=\esssup Z_1$. This implies that we have $Z_i\le x_{\max}$ almost surely for every $i$. If \(x_{\max}=0\), then \(Z_i=0\) a.s. for every \(i\). Hence all blocks are zero blocks, and by the zero-block convention both the single-grid and two-grid risks are zero for every \(g\). The conclusion is therefore immediate. We therefore consider the interesting case \(x_{\max}>0\).
Moreover, for every $\delta>0$,
\[
        p_\delta:=\mathbb{P}(Z_1>x_{\max}-\delta)>0.
\]
Thus
\[
        \mathbb{P}(M_g\le x_{\max}-\delta)
        =(1-p_\delta)^g\to0.
\]
As $M_g$ is nondecreasing in $g$, it follows that
\[
        M_g\to x_{\max}\qquad\text{a.s.}                    \tag{2}\label{eq:max-conv}
\]

Set
\[
        Y_i:=Z_i/x_{\max}\in[0,1].
\]
For $B\in\Cbar$, write
\[
        d_B(a):=\min_j|a-b_j|,
        \qquad
        \psi_B(a)=d_B(a)^2.
\]
Since $0\le d_B\le1$ and $d_B$ is $1$-Lipschitz,
\[
        |\psi_B(a)-\psi_B(a')|\le 2|a-a'|                  \tag{3}\label{eq:lipschitz-a}
\]
for all $a,a'\in[0,1]$, uniformly in $B$.  Also, if
$\|B-B'\|_\infty\le\eta$, then
\[
        \|\psi_B-\psi_{B'}\|_\infty\le 2\eta .             \tag{4}\label{eq:lipschitz-B}
\]
Indeed, the nearest-grid distance changes by at most $\eta$, and both distances
are bounded by one.  Since $\Cbar$ is compact, (4) implies that
$\{\psi_B:B\in\Cbar\}$ is totally bounded in the uniform norm.

We next prove the required uniform law of large numbers.  Fix $\eta>0$ and
choose $B^{(1)},\ldots,B^{(N)}\in\Cbar$ such that every $B\in\Cbar$ satisfies
$\|\psi_B-\psi_{B^{(k)}}\|_\infty\le\eta$ for some $k$.  By the strong law of
large numbers, simultaneously for all $k=1,\ldots,N$,
\[
        \frac1g\sum_{i=1}^g \psi_{B^{(k)}}(Y_i)
        \to \E\psi_{B^{(k)}}(Y_1)
        \qquad\text{a.s.}
\]
Approximating an arbitrary $B$ by a net point gives
\[
        \sup_{B\in\Cbar}\left|
        \frac1g\sum_{i=1}^g\psi_B(Y_i)-\E\psi_B(Y_1)
        \right|\to0
        \qquad\text{a.s.}                                  \tag{5}\label{eq:uniform-slln-fixed}
\]

It remains to replace the deterministic normalization $x_{\max}$ by the random
normalization $M_g$.  Since $x_{\max}>0$, (2) and the zero-block convention imply
that, almost surely, $M_g>0$ for all sufficiently large $g$.  For such $g$ and
all $i\le g$,
\[
        |A_{i,g}-Y_i|
        =\left|\frac{Z_i}{M_g}-\frac{Z_i}{x_{\max}}\right|
        \le \frac{x_{\max}-M_g}{x_{\max}},                  \tag{6}\label{eq:random-fixed-close}
\]
where we used $Z_i\le M_g\le x_{\max}$.  By (3), (6), and (2),
\[
        \sup_{B\in\Cbar}\left|
        \frac1g\sum_{i=1}^g\psi_B(A_{i,g})
        -\frac1g\sum_{i=1}^g\psi_B(Y_i)
        \right|\to0
        \qquad\text{a.s.}
\]
Combining this with (5),
\[
        \sup_{B\in\Cbar}\left|
        \frac1g\sum_{i=1}^g\psi_B(A_{i,g})
        -\E\psi_B(Y_1)
        \right|\to0
        \qquad\text{a.s.}                                  \tag{7}\label{eq:uniform-slln-random}
\]

Define
\[
        \ell_\infty(B):=x_{\max}^2\E\psi_B(Y_1)
        =x_{\max}^2\int_0^1\psi_B(a)\,d\mu_\infty(a).
\]
Using the reduction in \eqref{eq:scalar-reduction}, the bound
$0\le\psi_B\le1$, (2), and (7), we obtain
\[
        \Delta_g:=\sup_{B\in\Cbar}
        |L_B^{(g)}(X_1,\ldots,X_g)-\ell_\infty(B)|
        \to0
        \qquad\text{a.s.}                                  \tag{8}\label{eq:uniform-loss}
\]
Furthermore $0\le L_B^{(g)}\le x_{\max}^2$ and
$0\le\ell_\infty(B)\le x_{\max}^2$, so $0\le\Delta_g\le x_{\max}^2$.  Dominated
convergence yields
\[
        \E\Delta_g\to0.                                    \tag{9}\label{eq:L1-uniform-loss}
\]

Let
\[
        V_\infty:=\inf_{B\in\Cbar}\ell_\infty(B).
\]
The infimum map is $1$-Lipschitz under uniform perturbations; hence
\[
        |R_g^\star-V_\infty|
        \le \E\Delta_g\to0.                                \tag{10}\label{eq:single-risk-limit}
\]
For the two-grid objective, use
\[
        |\min\{a,b\}-\min\{c,d\}|
        \le \max\{|a-c|,|b-d|\}.
\]
Together with (8)--(9), this gives
\[
        R_{\mathrm{PO2},g}^\star
        \to
        W_\infty
        :=\inf_{B_1,B_2\in\Cbar}
          \min\{\ell_\infty(B_1),\ell_\infty(B_2)\}.         \tag{11}\label{eq:two-risk-limit}
\]
But $W_\infty=V_\infty$: the left-hand side is at least $V_\infty$ because each
$\ell_\infty(B_j)\ge V_\infty$, and it is at most $V_\infty$ by choosing both
grids along a minimizing sequence for $\ell_\infty$.  Therefore both risks have
the same limit, proving
\[
        R_g^\star-R_{\mathrm{PO2},g}^\star\to0.
\]

Finally, the same conclusion holds for the strict class $\Cstrict$.  For any
$B\in\Cbar$ and $\rho\in(0,1)$, define
\[
        b_j^{(\rho)}:=(1-\rho)b_j+\rho\,j/7,
        \qquad j=0,\ldots,7.
\]
Then $B^{(\rho)}\in\Cstrict$ and $B^{(\rho)}\to B$.  By (4), the corresponding
losses and limiting losses converge uniformly.  Hence the single-grid and
two-grid infima over $\Cstrict$ equal the infima over $\Cbar$, both at finite
$g$ and in the limiting problem.
\end{proof}

\section{Exact Grid Values}
\label{app:grids}

\paragraph{MPO2 (FP8-snapped).}
\[
\begin{aligned}
B_1&=\{-1,-0.8125,-0.625,-0.5,-0.375,-0.28125,-0.171875,-0.0703125,\\
&\qquad 0.015625,0.109375,0.21875,0.34375,0.46875,0.625,0.75,1\},\\
B_2&=\{-1,-0.75,-0.5625,-0.4375,-0.3125,-0.203125,-0.109375,-0.015625,\\
&\qquad 0.0703125,0.171875,0.28125,0.40625,0.5,0.6875,0.875,1\}.
\end{aligned}
\]

\paragraph{Split87 primary (FP8-snapped).}
\[
\begin{aligned}
\{&-1,-0.8125,-0.625,-0.46875,-0.34375,-0.234375,-0.140625,-0.0546875,\\
&\;0,\;0.0625,\;0.171875,\;0.28125,\;0.40625,\;0.5625,\;0.75,\;1\}.
\end{aligned}
\]

\paragraph{SFP4.}
Grid~A uses the standard E2M1 values $\{0,\pm0.5,\pm1,\pm1.5,\pm2,\pm3,\pm4,\pm6\}$.
Grids~$B^+$ and $B^-$ shift the E2M1 grid by $\pm c$ with $c=0.5$.

\begin{figure}[h]
\centering
\begin{subfigure}[b]{0.48\textwidth}
\includegraphics[width=\textwidth]{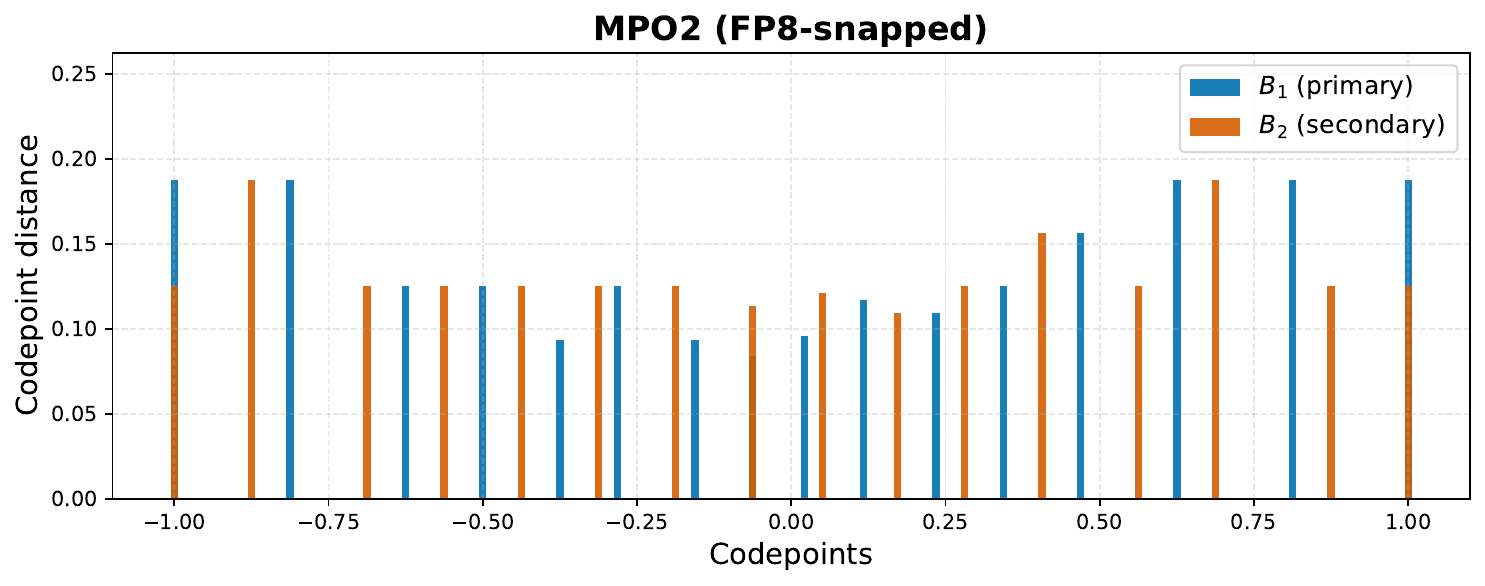}
\caption{MPO2 (FP8-snapped)}
\end{subfigure}
\hfill
\begin{subfigure}[b]{0.48\textwidth}
\includegraphics[width=\textwidth]{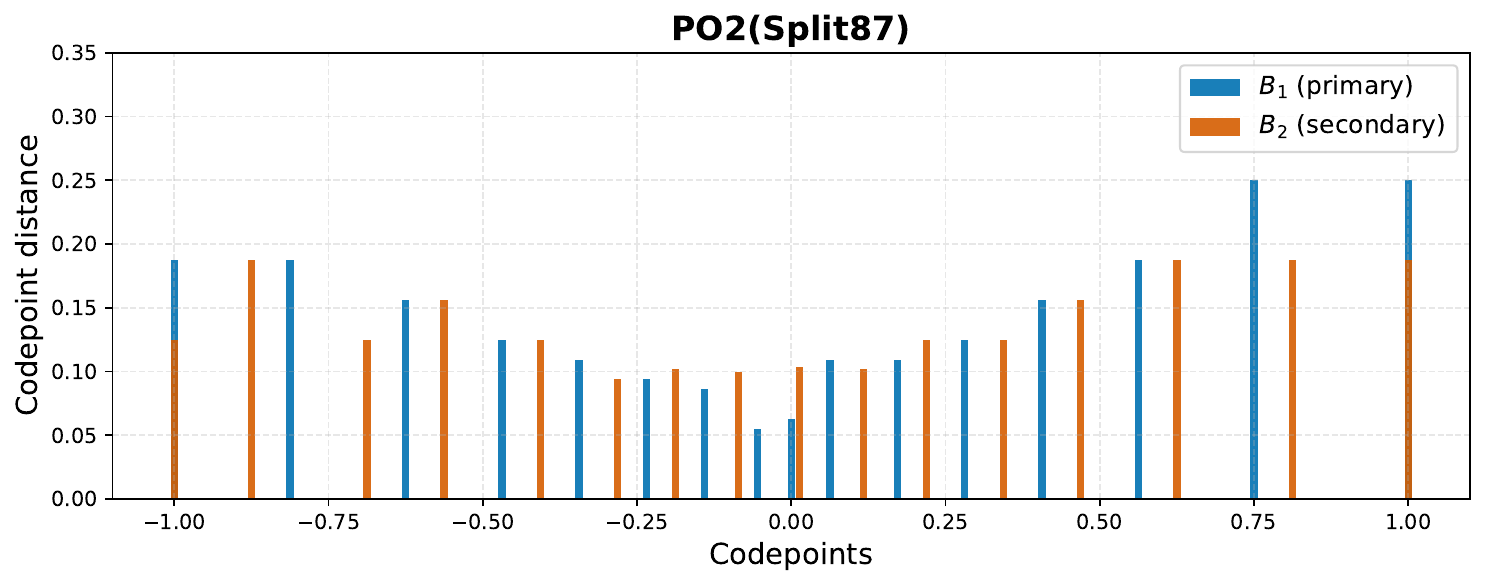}
\caption{PO2(Split87)}
\end{subfigure}

\vspace{0.5em}
\begin{subfigure}[b]{0.48\textwidth}
\includegraphics[width=\textwidth]{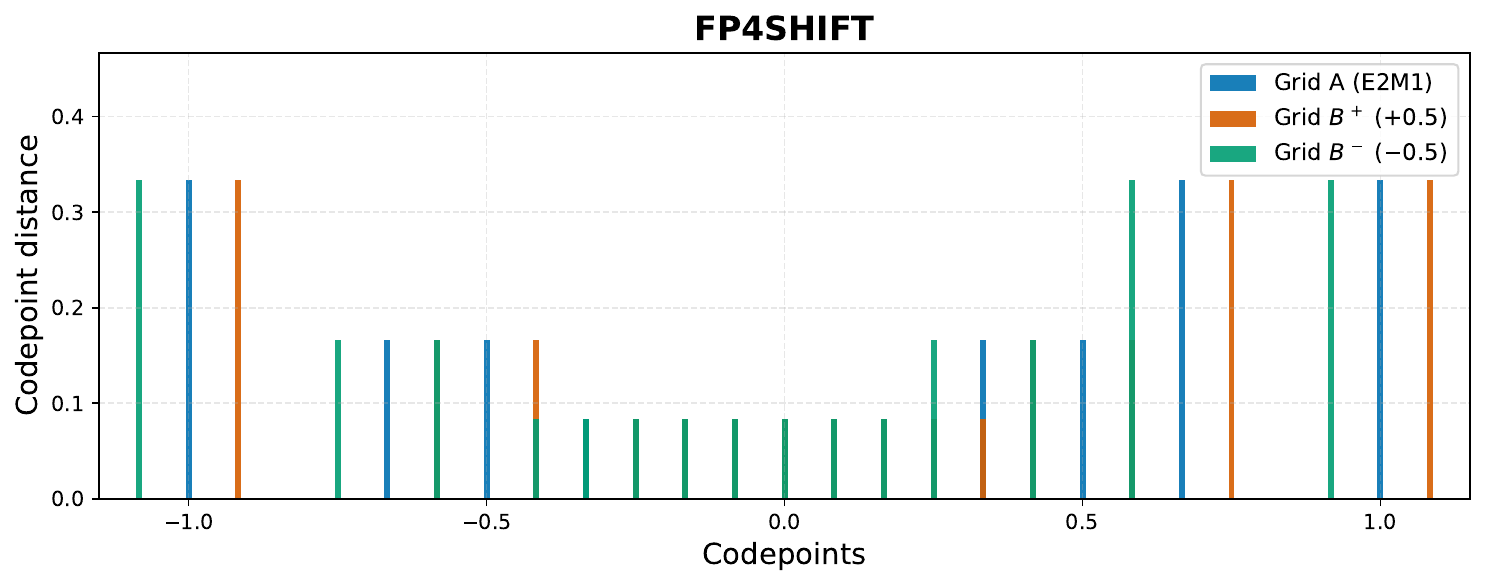}
\caption{FP4SHIFT (three grids)}
\end{subfigure}
\hfill
\begin{subfigure}[b]{0.48\textwidth}
\includegraphics[width=\textwidth]{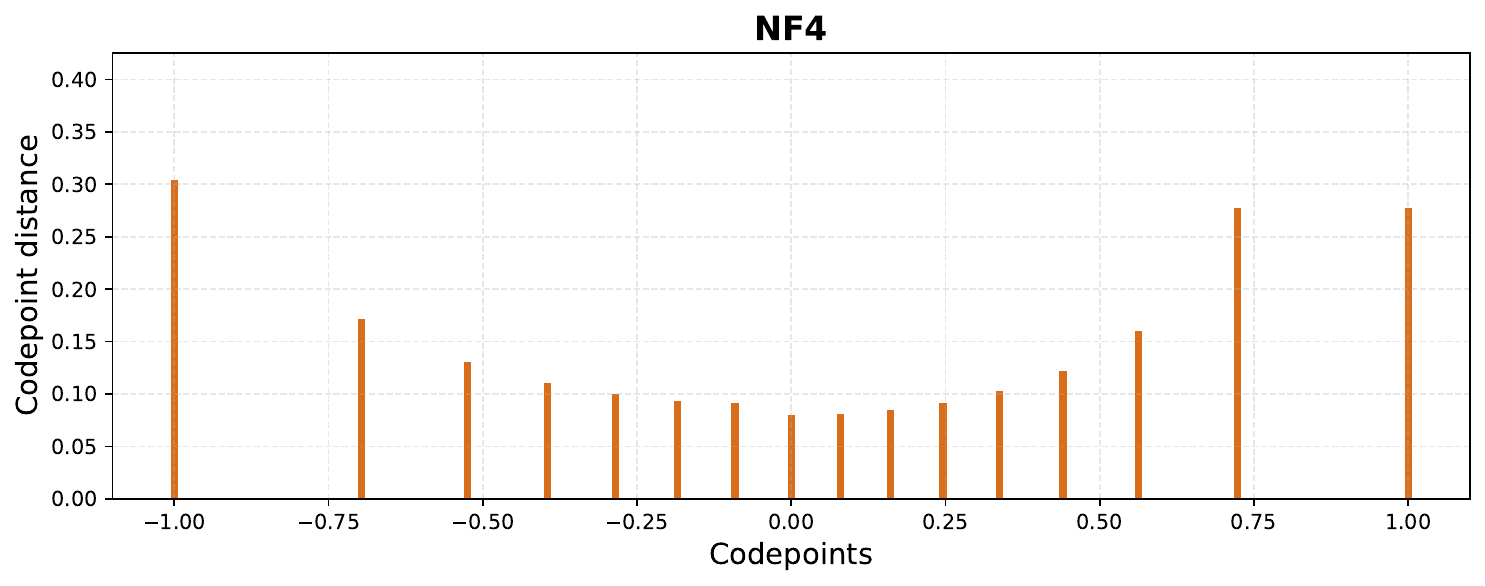}
\caption{NF4 (single grid)}
\end{subfigure}
\caption{Codepoint spacing for PO2 grid families. Each bar is placed at a codepoint's position; its height shows the distance to the next codepoint. Dual-grid pairs interleave their codepoints, increasing resolution in the dense near-zero region.}
\label{fig:grid_vis}
\end{figure}

\section{Additional PTQ results}
\label{app:ptq_results}

Results in the main paper use greedy decoding with a single evaluation run. To show the variance, we repeat experiments on  Qwen3.5-4B  across four seeds and report the resulting spread in Figure~\ref{fig:seed_run}.

\begin{figure}
    \centering
    \includegraphics[width=\linewidth]{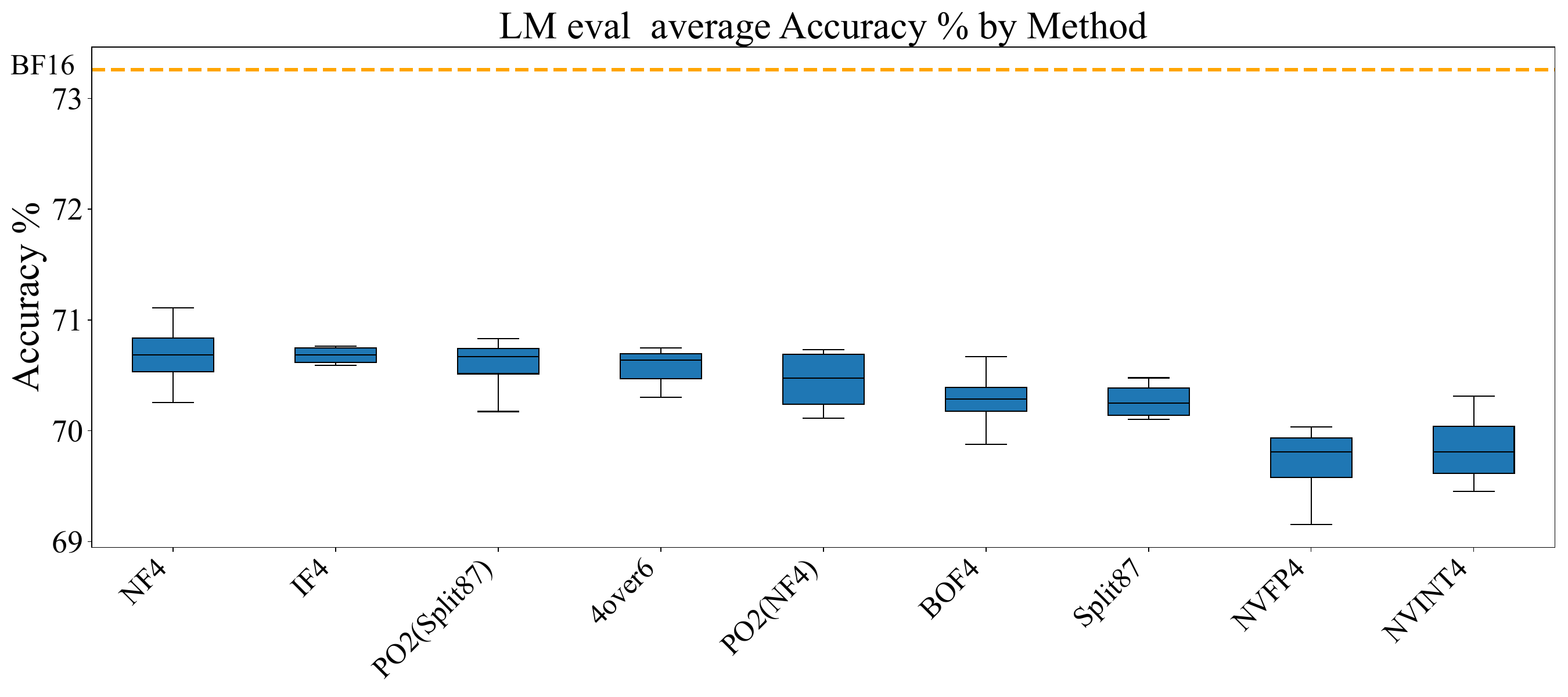}
    \caption{Qwen3.5-4B multiple seed runs.}
    \label{fig:seed_run}
\end{figure}

Tables~\ref{tab:qwen354b-wo-transforms} and~\ref{tab:qwen354b-wa-transforms} compare the identity and Hadamard transforms on Qwen3.5-4B for weight-only and weight+activation quantization, respectively.

\begin{table*}[t]
\centering
\caption{Weight-only quantization (W4A16) on \textbf{Qwen3.5-4B}, comparing the \emph{Identity} (no transform) and \emph{Hadamard} pre-rotation. KL columns report the KL divergence between the quantized and BF16 output distributions on the indicated calibration set ($\downarrow$); EAR$_{\text{top10},\,C4}$ is the exact-agreement rate of the top-10 logits on C4 ($\uparrow$). \emph{Avg} is the mean over Winogrande, ARC-C, ARC-E, Lambada (standard), PIQA, Hellaswag (10-shot), MMLU, IFEval (Prompt), and GSM8K-CoT. Best result per column within each transform block is in \textbf{bold}.}
\label{tab:qwen354b-wo-transforms}
\setlength{\tabcolsep}{4pt}
\resizebox{\textwidth}{!}{%
\begin{tabular}{l ccccc ccccc}
\toprule
 & \multicolumn{5}{c}{\textbf{Identity transform}} & \multicolumn{5}{c}{\textbf{Hadamard transform}} \\
\cmidrule(lr){2-6} \cmidrule(lr){7-11}
Method & KL$_{\text{t10},C4}\!\downarrow$ & EAR$_{\text{t10},C4}\!\uparrow$ & KL$_{\text{W2}}\!\downarrow$ & KL$_{\text{C4}}\!\downarrow$ & Avg$\!\uparrow$ & KL$_{\text{t10},C4}\!\downarrow$ & EAR$_{\text{t10},C4}\!\uparrow$ & KL$_{\text{W2}}\!\downarrow$ & KL$_{\text{C4}}\!\downarrow$ & Avg$\!\uparrow$ \\
\midrule
\textit{BF16 (baseline)} & \textit{0.0000} & \textit{1.0000} & \textit{0.0000} & \textit{0.0000} & \textit{73.26} & \textit{0.0000} & \textit{1.0000} & \textit{0.0000} & \textit{0.0000} & \textit{73.26} \\
\midrule
PO2(NF4) & 0.0200 & 0.9429 & 0.0534 & 0.0258 & 72.17 & 0.0204 & 0.9420 & 0.0521 & 0.0264 & \textbf{72.63} \\
MPO2 & 0.0211 & 0.9411 & 0.0554 & 0.0275 & \textbf{72.67} & \textbf{0.0199} & \textbf{0.9427} & 0.0519 & 0.0259 & 71.95 \\
IF4 & 0.0247 & 0.9361 & 0.0721 & 0.0322 & 72.21 & 0.0251 & 0.9357 & 0.0613 & 0.0326 & 72.37 \\
PO2(Split87) & \textbf{0.0191} & \textbf{0.9437} & \textbf{0.0511} & \textbf{0.0246} & 72.24 & 0.0201 & 0.9426 & \textbf{0.0506} & \textbf{0.0259} & 72.24 \\
NF4 & 0.0246 & 0.9365 & 0.0588 & 0.0319 & 72.36 & 0.0279 & 0.9324 & 0.0675 & 0.0364 & 72.03 \\
Split87 & 0.0223 & 0.9392 & 0.0527 & 0.0291 & 72.12 & 0.0259 & 0.9347 & 0.0679 & 0.0335 & 72.25 \\
NVFP4 & 0.0308 & 0.9287 & 0.0792 & 0.0401 & 71.34 & 0.0370 & 0.9219 & 0.0835 & 0.0486 & 72.01 \\
NVINT4 & 0.0328 & 0.9263 & 0.0810 & 0.0431 & 71.60 & 0.0305 & 0.9291 & 0.0769 & 0.0399 & 71.74 \\
\bottomrule
\end{tabular}%
}
\end{table*}

\begin{table*}[t]
\centering
\caption{Weight and activation quantization (W4A4) on \textbf{Qwen3.5-4B}, comparing the \emph{Identity} (no transform) and \emph{Hadamard} pre-rotation. KL columns report the KL divergence between the quantized and BF16 output distributions on the indicated calibration set ($\downarrow$, lower is better); EAR$_{\text{top10},\,C4}$ is the exact-agreement rate of the top-10 logits on C4 ($\uparrow$). \emph{Avg} is the mean over Winogrande, ARC-C, ARC-E, Lambada (standard), PIQA, Hellaswag (10-shot), MMLU, IFEval (Prompt), and GSM8K-CoT. Best result per column within each transform block is in \textbf{bold}. }
\label{tab:qwen354b-wa-transforms}
\setlength{\tabcolsep}{4pt}
\resizebox{\textwidth}{!}{%
\begin{tabular}{l ccccc ccccc}
\toprule
 & \multicolumn{5}{c}{\textbf{Identity transform}} & \multicolumn{5}{c}{\textbf{Hadamard transform}} \\
\cmidrule(lr){2-6} \cmidrule(lr){7-11}
Method & KL$_{\text{t10},C4}\!\downarrow$ & EAR$_{\text{t10},C4}\!\uparrow$ & KL$_{\text{W2}}\!\downarrow$ & KL$_{\text{C4}}\!\downarrow$ & Avg$\!\uparrow$ & KL$_{\text{t10},C4}\!\downarrow$ & EAR$_{\text{t10},C4}\!\uparrow$ & KL$_{\text{W2}}\!\downarrow$ & KL$_{\text{C4}}\!\downarrow$ & Avg$\!\uparrow$ \\
\midrule
\textit{BF16 (baseline)} & \textit{0.0000} & \textit{1.0000} & \textit{0.0000} & \textit{0.0000} & \textit{73.26} & \textit{0.0000} & \textit{1.0000} & \textit{0.0000} & \textit{0.0000} & \textit{73.26} \\
\midrule
PO2(NF4) & 0.0406 & 0.9166 & 0.0919 & 0.0525 & 70.68 & 0.0365 & 0.9214 & 0.0851 & 0.0471 & 71.38 \\
Split87 & 0.0438 & 0.9132 & 0.0957 & 0.0571 & 70.48 & 0.0477 & 0.9100 & 0.1073 & 0.0619 & \textbf{71.52} \\
IF4 & 0.0471 & 0.9097 & 0.1125 & 0.0616 & 70.76 & 0.0425 & 0.9150 & 0.0955 & 0.0552 & 71.23 \\
PO2(Split87) & \textbf{0.0376} & \textbf{0.9195} & \textbf{0.0869} & \textbf{0.0487} & 70.71 & 0.0360 & 0.9220 & 0.0821 & 0.0464 & 71.22 \\
NF4 & 0.0490 & 0.9083 & 0.1054 & 0.0637 & 71.11 & 0.0528 & 0.9053 & 0.1160 & 0.0687 & 70.49 \\
MPO2 & 0.0502 & 0.9073 & 0.1092 & 0.0666 & 70.81 & \textbf{0.0344} & \textbf{0.9235} & \textbf{0.0798} & \textbf{0.0447} & 70.72 \\
4over6 & 0.0516 & 0.9059 & 0.1112 & 0.0670 & 70.64 & 0.0578 & 0.9010 & 0.1201 & 0.0752 & 70.31 \\
NVINT4 & 0.0714 & 0.8872 & 0.1558 & 0.0957 & 70.31 & 0.0505 & 0.9074 & 0.1158 & 0.0663 & 70.23 \\
NVFP4 & 0.0586 & 0.8993 & 0.1305 & 0.0767 & 69.72 & 0.0712 & 0.8895 & 0.1459 & 0.0935 & 70.57 \\
\bottomrule
\end{tabular}%
}
\end{table*}

\section{Additional QAT Details and Results}
\label{app:qat}

\subsection{Hardware Requirements}
\label{app:hardware}

The QAT experiments, as well as the majority of PTQ experiments, were performed on NVIDIA Hopper or Blackwell nodes (8xH100 or 8xB200). The workloads did not require parallelism beyond singular nodes.

\subsection{Model Hyper-Parameters}

For QAT experiments, we follow the setup of~\citet{panferov2025queststabletrainingllms}. We record the model-specific hyper-parameters in Table~\ref{tab:model_params} and common pre-training hyper-parameters in Table~\ref{tab:common_params}.

\begin{table}[t]
    \centering
    \caption{Model-specific QAT hyper-parameters.}
    \label{tab:model_params}
    \begin{tabular}{lcccc}
        \toprule
        \textbf{Hyperparameter} & \textbf{30M} & \textbf{50M} & \textbf{100M} \\
        \midrule
        Number of Layers & 6 & 7 & 8 \\
        Embedding Dimension & 640 & 768 & 1024 \\
        Attention Heads & 5 & 6 & 8 \\
        Learning Rate & 0.0012 & 0.0012 & 0.0009 \\
        \bottomrule
    \end{tabular}
\end{table}

\begin{table}[t]
    \centering
    \caption{Common hyper-parameters used across all model sizes for QAT.}
    \label{tab:common_params}
    \begin{tabular}{lc}
        \toprule
        \textbf{Hyperparameter} & \textbf{Value} \\
        \midrule
        Sequence Length & 512 \\
        Batch Size & 512 \\
        Optimizer & AdamW \\
        Learning Rate Schedule & Cosine, 10\% warm-up \\
        Gradient Clipping & 1.0 \\
        Weight Decay ($\gamma$) & 0.1 \\
        Number of GPUs & 8 \\
        Data Type (optimizer/accumulators) & \texttt{FP32} \\
        \bottomrule
    \end{tabular}
\end{table}

\subsection{Weight-Only QAT Experiments}

\begin{figure}[t]
    \centering
    \includegraphics[width=1.0\linewidth]{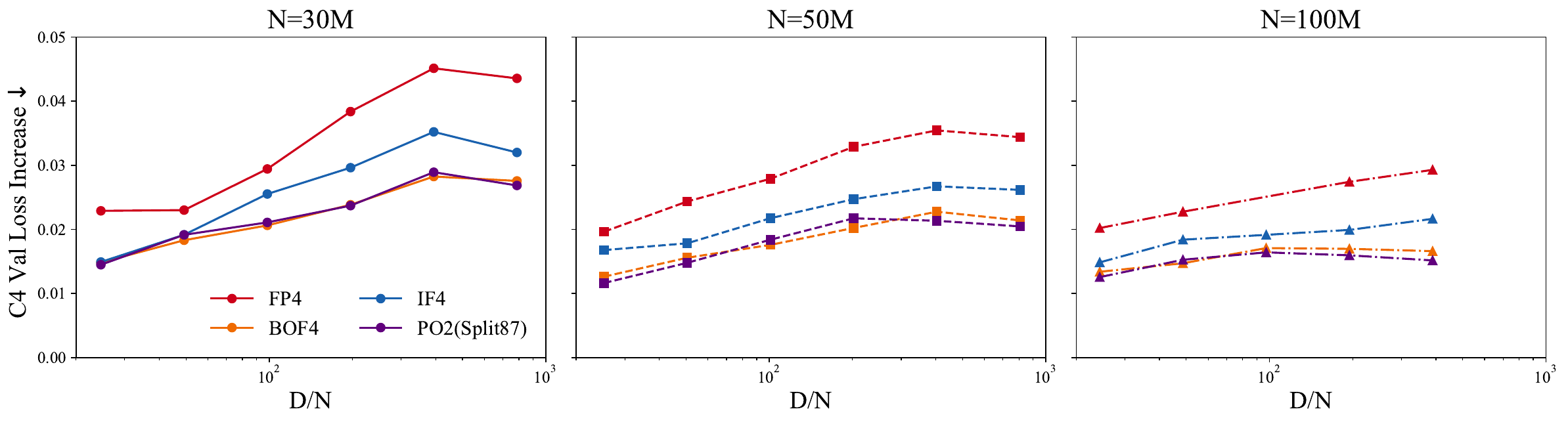}
    \caption{Weight-only QAT final C4 validation loss gap relative to BF16 pre-training.}
    \label{fig:qat_wo}
\end{figure}

Additionally, we validate multi-grid quantization for weight-only QAT. The results, shown in Figure~\ref{fig:qat_wo}, generally agree with the full QAT results, but multiple grids show lesser improvement.

\section{Kernels for Current Hardware}
\subsection{SFP4}
\label{sec:sfp4-kernels}
For the FPShift approach, we do not need a new main matmul kernel, but instead a set of pre- and post-processing kernels.
In pre-processing, the encoded weight scale-factors are decoded into \textquote{clean} scales (with the group-bits cleared)
and shifts, i.e., scales times shift direction. For activations, we pre-calculate the sum over every group.
After the main matmul, we run the correction kernel which multiplies the shifts and group sums, another gemm of $1/16$
the size of the main gemm, but running at higher precision.

\begin{table}[tbph]
\centering
\caption{Benchmark results on RTX 5090 for SFP4.}
\label{tab:sfp4-benchmark}
\begin{tabular}{lrrr}
\toprule
\textbf{Metric / M,N,K} & \textbf{4096, 5120, 5120} & \textbf{4096, 5120, 27648} & \textbf{4096, 27648, 5120} \\
\midrule
\multicolumn{4}{l}{\textit{Kernel Performance}} \\
nvfp4\_gemm       & 0.1464 ms & 0.7905 ms & 0.7979 ms \\
baseline          & 0.2068 ms & 1.0364 ms & 0.8448 ms \\
SFP4              & 0.3199 ms          & 1.5340 ms          & 1.3603 ms          \\
$\quad$rel. overhead   & +54.7\%          & +48.0\%          & +61.0\%          \\
\midrule
\multicolumn{4}{l}{\textit{Non-matmul Breakdown}} \\
act. quant             & 0.0591 ms & 0.2737 ms & 0.0591 ms \\
weight\_preprocess        & 0.0061 ms & 0.0287 ms & 0.0266 ms \\
activation\_sum   & 0.0243 ms & 0.1037 ms & 0.0243 ms \\
correction        & 0.0796 ms & 0.3294 ms & 0.3724 ms \\
\bottomrule
\end{tabular}
\end{table}

Benchmark results on an RTX 5090 are presented in \autoref{tab:sfp4-benchmark}.
Compared to the baseline, which includes both the raw nvfp4 matmul and the activation quantization kernel, we see an overhead of about 50-60\%. We also show the breakdown of the overhead over pure matmul into quantization, preprocessing, and post-processing kernels.
For shapes with large inputs and small outputs (down-projection), the input-quantization is expensive, and as such the relative cost of the correction matmul is comparatively small. For small inputs and large outputs, on the other hand, input processing is cheap and most overhead comes from the correction.

\subsection{General PO2 kernels}
\label{sec:po2-kernels}
We also provide two kernels that allow running general PO2-quantized matmuls on current hardware.
The first targets the case of memory-bound decoding. As such, it assumes its activation input
is small and there would be no point in quantizing it to lower bit-width. It accepts two lookup
tables of 16 BF16-entries each; as CUDA-cores do not have arithmetic units for FP8 numbers, this
is actually more efficient than having lower-bit entries in this scenario. Note that even at
16-bit per entry, the entire table is just $32 \times 2$ bytes, fitting in a contiguous address range
that spans only 16 shared memory banks. Thus, even if the threads of each warp perform completely random
access into the table, there are no bank conflicts that need to be serialized.
For an $m \times n \times k$ matmul, each inner product (k-dimension) is handled collectively by one warp,
and each thread computes $m$ (batch dimension) results in parallel, where $m \in [1, 8]$ is a compile-time
constant.

For M=1, N=27648, K=5120, we measure a running time of 50 µs, corresponding to a memory throughput of 91\% of peak bandwidth on an RTX 5090. Increasing to $M=2$ is almost free at 52 µs, $M=4$ at 57 µs, but for $M=8$ increased register pressure leads to lowered occupancy, and we only achieve 58\% memory throughput with 92 µs running time. 

On the other hand, for large $m$ it is necessary to use tensor cores for good performance. Here, we can show
that if the shapes are sufficiently nice (multiples of 128), one can use a large $128\times128$ per-block
tile so that weight values decoded into shared memory can be reused sufficiently often. In that way, the additional
work of software-decoding PO2 can be hidden almost perfectly behind the bf16 \texttt{mma} tensorcore math.

Using the same $M=27648$ $K=5120$ configuration as above, but with a larger batch size $M=256$, the tensor-core kernel requires 555 µs running time at 55\% tensor-core utilization. This number goes up to 80\% for $M=4096$ with a running time of 5.8 ms.
Running a pure BF16 matmul at the same shape costs 4.8 ms, so the lookup-table overhead is about 20\%.

\section{Format Snapping}
\label{sec:universality}



\label{sec:format-snapping}
FP8 E4M3 snapping is nearly free for 16-level grids because the representable values in $[-1,1]$ are dense relative to the codebook size.
Table~\ref{tab:format_snapping} confirms this: FP8 is neutral or slightly beneficial on Qwen3-8B and Qwen3.5-9B, and costs 12\% relative KL on Mistral-7B.
FP6 E3M2 is more restrictive, costing 10--20\% relative KL depending on the model.

\begin{table}[b]
\centering
\caption{MPO2 format snapping: weight-only WikiText-2 KL divergence.}
\label{tab:format_snapping}
\begin{tabular}{lccc}
\toprule
Model & Raw MPO2 & MPO2-FP8 & MPO2-FP6 \\
\midrule
Qwen3-8B & 0.0467 & 0.0452 & 0.0544 \\
Qwen3.5-9B & 0.0642 & 0.0639 & 0.0715 \\
Mistral-7B-v0.3 & 0.0137 & 0.0153 & 0.0157 \\
\bottomrule
\end{tabular}
\end{table}

\section{Full downstream task results}
\label{app:full_tables}

We provide KL divergences, EAR\cite{helcig2026statisticallylosslessquantizationlargelanguage} and full downstream task results for each model in Tables~\ref{tab:wa-llama323b},\ref{tab:wa-qwen38b},\ref{tab:wa-qwen314b},\ref{tab:wa-qwen352b},\ref{tab:wa-qwen354b},\ref{tab:wa-qwen359b},\ref{tab:wa-qwen3527b}.

\begin{table*}[t]
\centering
\caption{Weight and activation quantization (W4A4) on \textbf{Llama-3.2-3B-Instruct}. KL columns report the KL divergence between the quantized and BF16 output distributions on the indicated calibration set ($\downarrow$); EAR$_{\text{top10},C4}$ (Expected Acceptance Rate), the maximum token-agreement probability under optimal coupling on C4 ($\uparrow$) Benchmark columns: Winogrande (WG), ARC-C, ARC-E, Lambada-standard (LMB), PIQA, Hellaswag 10-shot (HS10), MMLU, IFEval-Prompt (IFE-P), and GSM8K-CoT. \emph{Avg} is their mean and \emph{Recovery} is \emph{Avg}\,$\div$\,BF16. Best non-BF16 result per column is in \textbf{bold}.}
\label{tab:wa-llama323b}
\setlength{\tabcolsep}{3pt}
\resizebox{\textwidth}{!}{%
\begin{tabular}{l cccc ccccccccc c c}
\toprule
 & \multicolumn{4}{c}{\textbf{Distribution-level}} & \multicolumn{9}{c}{\textbf{Benchmarks}} & & \\
\cmidrule(lr){2-5} \cmidrule(lr){6-14}
Method & KL$_{\text{t10},C4}\!\downarrow$ & EAR$_{\text{t10},C4}\!\uparrow$ & KL$_{\text{W2}}\!\downarrow$ & KL$_{\text{C4}}\!\downarrow$ & WG & ARC-C & ARC-E & LMB & PIQA & HS10 & MMLU & IFE-P & GSM8K & Avg & \textbf{Recovery} \\
\midrule
\textit{BF16 (baseline)} & \textit{0.0000} & \textit{1.0000} & \textit{0.0000} & \textit{0.0000} & \textit{71.03} & \textit{71.30} & \textit{45.99} & \textit{61.25} & \textit{76.99} & \textit{73.58} & \textit{62.21} & \textit{70.98} & \textit{77.94} & \textit{67.92} & \textit{100.00\%} \\
\midrule
PO2(Split87) & \textbf{0.0467} & \textbf{0.9122} & \textbf{0.0647} & \textbf{0.0615} & 68.19 & 70.41 & 45.90 & 59.60 & 76.39 & \textbf{72.25} & 60.27 & \textbf{69.87} & \textbf{74.60} & \textbf{66.39} & \textbf{97.75\%} \\
PO2(NF4) & 0.0505 & 0.9068 & 0.0710 & 0.0677 & 68.35 & \textbf{71.68} & 46.08 & 58.39 & 76.77 & 71.28 & 60.05 & 68.58 & 73.69 & 66.10 & 97.32\% \\
NF4 & 0.0575 & 0.9015 & 0.0778 & 0.0774 & 67.80 & 69.87 & 45.31 & 60.08 & \textbf{77.31} & 71.78 & 59.20 & 69.32 & 70.96 & 65.74 & 96.79\% \\
Split87 & 0.0516 & 0.9060 & 0.0734 & 0.0698 & 68.98 & 69.87 & 44.54 & 58.99 & 76.77 & 72.00 & \textbf{60.35} & 68.21 & 71.42 & 65.68 & 96.71\% \\
SFP4 & 0.0612 & 0.8988 & 0.0850 & 0.0799 & 68.82 & 69.15 & 45.90 & \textbf{60.72} & 76.39 & 71.05 & 59.49 & 66.91 & 72.02 & 65.61 & 96.60\% \\
4over6 & 0.0610 & 0.8990 & 0.0856 & 0.0800 & 69.22 & 68.94 & 44.80 & 58.65 & 76.28 & 71.65 & 59.23 & 67.10 & 73.77 & 65.51 & 96.46\% \\
IF4 & 0.0556 & 0.9038 & 0.0751 & 0.0733 & 67.96 & 68.06 & 44.37 & 59.52 & 76.12 & 71.63 & 59.48 & 66.73 & 73.92 & 65.31 & 96.16\% \\
BOF4 & 0.0572 & 0.9024 & 0.0776 & 0.0761 & 67.48 & 70.03 & 44.54 & 58.59 & 76.99 & 71.77 & 59.93 & 65.25 & 72.63 & 65.25 & 96.07\% \\
NVFP4 & 0.0676 & 0.8931 & 0.0943 & 0.0895 & 67.72 & 70.03 & 43.94 & 58.43 & 75.41 & 71.37 & 58.84 & 67.84 & 71.80 & 65.04 & 95.77\% \\
NVINT4 & 0.0820 & 0.8816 & 0.1160 & 0.1118 & 67.25 & 67.30 & 44.20 & 56.72 & 76.12 & 71.22 & 57.99 & 64.88 & 70.36 & 64.00 & 94.24\% \\
\bottomrule
\end{tabular}%
}
\end{table*}

\begin{table*}[t]
\centering
\caption{Weight and activation quantization (W4A4) on \textbf{Qwen3-8B}. KL columns report the KL divergence between the quantized and BF16 output distributions on the indicated calibration set ($\downarrow$); EAR$_{\text{top10},C4}$ (Expected Acceptance Rate), the maximum token-agreement probability under optimal coupling on C4 ($\uparrow$). Benchmark columns: Winogrande (WG), ARC-C, ARC-E, Lambada-standard (LMB), PIQA, Hellaswag 10-shot (HS10), MMLU, IFEval-Prompt (IFE-P), and GSM8K-CoT. \emph{Avg} is their mean and \emph{Recovery} is \emph{Avg}\,$\div$\,BF16. Best non-BF16 result per column is in \textbf{bold}.}
\label{tab:wa-qwen38b}
\setlength{\tabcolsep}{3pt}
\resizebox{\textwidth}{!}{%
\begin{tabular}{l cccc ccccccccc c c}
\toprule
 & \multicolumn{4}{c}{\textbf{Distribution-level}} & \multicolumn{9}{c}{\textbf{Benchmarks}} & & \\
\cmidrule(lr){2-5} \cmidrule(lr){6-14}
Method & KL$_{\text{t10},C4}\!\downarrow$ & EAR$_{\text{t10},C4}\!\uparrow$ & KL$_{\text{W2}}\!\downarrow$ & KL$_{\text{C4}}\!\downarrow$ & WG & ARC-C & ARC-E & LMB & PIQA & HS10 & MMLU & IFE-P & GSM8K & Avg & \textbf{Recovery} \\
\midrule
\textit{BF16 (baseline)} & \textit{0.0000} & \textit{1.0000} & \textit{0.0000} & \textit{0.0000} & \textit{70.40} & \textit{80.77} & \textit{56.14} & \textit{61.11} & \textit{77.97} & \textit{76.54} & \textit{72.90} & \textit{80.04} & \textit{91.13} & \textit{74.11} & \textit{100.00\%} \\
\midrule
PO2(NF4) & 0.0473 & 0.9173 & 0.0799 & 0.0562 & 70.01 & \textbf{81.57} & 55.89 & 59.54 & 76.39 & 75.32 & 71.09 & 81.15 & \textbf{89.76} & \textbf{73.41} & \textbf{99.06\%} \\
BOF4 & 0.0486 & 0.9158 & 0.0805 & 0.0574 & 70.24 & 79.46 & 57.17 & \textbf{60.22} & 77.15 & 74.98 & 71.12 & 81.52 & 88.10 & 73.33 & 98.94\% \\
PO2(Split87) & \textbf{0.0412} & \textbf{0.9220} & \textbf{0.0689} & \textbf{0.0487} & \textbf{70.40} & 79.92 & 55.29 & 59.38 & 76.22 & 75.06 & \textbf{71.57} & 81.52 & 88.93 & 73.15 & 98.70\% \\
SFP4 & 0.0546 & 0.9115 & 0.0912 & 0.0645 & 69.53 & 80.81 & 55.55 & 59.11 & 76.28 & 74.43 & 70.78 & 82.07 & 88.93 & 73.05 & 98.57\% \\
NF4 & 0.0514 & 0.9136 & 0.0867 & 0.0611 & 69.14 & 79.67 & 56.14 & 58.84 & 76.44 & 75.20 & 71.44 & 81.33 & 87.95 & 72.91 & 98.38\% \\
Split87 & 0.0457 & 0.9185 & 0.0761 & 0.0535 & 69.85 & 78.79 & 55.20 & 59.15 & 77.26 & 75.04 & 71.54 & 81.33 & 87.87 & 72.89 & 98.36\% \\
NVFP4 & 0.0600 & 0.9062 & 0.0994 & 0.0720 & 69.69 & 79.59 & 54.78 & 58.68 & \textbf{77.42} & 74.89 & 71.00 & 81.33 & 87.57 & 72.77 & 98.19\% \\
IF4 & 0.0507 & 0.9130 & 0.0855 & 0.0616 & 69.38 & 78.37 & 53.58 & 59.73 & 75.90 & 74.99 & 70.72 & \textbf{83.18} & 88.86 & 72.74 & 98.16\% \\
4over6 & 0.0544 & 0.9115 & 0.0900 & 0.0647 & 68.98 & 79.08 & 55.29 & 59.42 & 76.71 & 74.54 & 71.29 & 80.41 & 88.48 & 72.69 & 98.08\% \\
NVINT4 & 0.0768 & 0.8916 & 0.1276 & 0.0942 & 70.01 & 75.97 & 53.33 & 57.31 & 76.93 & 74.47 & 69.55 & 80.41 & 88.40 & 71.82 & 96.91\% \\
\bottomrule
\end{tabular}%
}
\end{table*}

\begin{table*}[t]
\centering
\caption{Weight and activation quantization (W4A4) on \textbf{Qwen3-14B}. KL columns report the KL divergence between the quantized and BF16 output distributions on the indicated calibration set ($\downarrow$); EAR$_{\text{top10},C4}$ (Expected Acceptance Rate), the maximum token-agreement probability under optimal coupling on C4 ($\uparrow$) Benchmark columns: Winogrande (WG), ARC-C, ARC-E, Lambada-standard (LMB), PIQA, Hellaswag 10-shot (HS10), MMLU, IFEval-Prompt (IFE-P), and GSM8K-CoT. \emph{Avg} is their mean and \emph{Recovery} is \emph{Avg}\,$\div$\,BF16. Best non-BF16 result per column is in \textbf{bold}.}
\label{tab:wa-qwen314b}
\setlength{\tabcolsep}{3pt}
\resizebox{\textwidth}{!}{%
\begin{tabular}{l cccc ccccccccc c c}
\toprule
 & \multicolumn{4}{c}{\textbf{Distribution-level}} & \multicolumn{9}{c}{\textbf{Benchmarks}} & & \\
\cmidrule(lr){2-5} \cmidrule(lr){6-14}
Method & KL$_{\text{t10},C4}\!\downarrow$ & EAR$_{\text{t10},C4}\!\uparrow$ & KL$_{\text{W2}}\!\downarrow$ & KL$_{\text{C4}}\!\downarrow$ & WG & ARC-C & ARC-E & LMB & PIQA & HS10 & MMLU & IFE-P & GSM8K & Avg & \textbf{Recovery} \\
\midrule
\textit{BF16 (baseline)} & \textit{0.0000} & \textit{1.0000} & \textit{0.0000} & \textit{0.0000} & \textit{74.51} & \textit{82.74} & \textit{60.49} & \textit{64.29} & \textit{80.03} & \textit{79.80} & \textit{77.18} & \textit{85.40} & \textit{92.04} & \textit{77.39} & \textit{100.00\%} \\
\midrule
IF4 & 0.0382 & 0.9236 & 0.0626 & 0.0470 & 73.16 & 81.94 & 59.30 & 63.54 & 79.65 & 79.10 & \textbf{76.46} & \textbf{85.58} & 92.49 & \textbf{76.80} & \textbf{99.24\%} \\
Split87 & 0.0337 & 0.9291 & 0.0535 & 0.0397 & 73.88 & 82.24 & 60.07 & 63.03 & 79.00 & \textbf{79.22} & 75.48 & 85.21 & 91.74 & 76.65 & 99.05\% \\
4over6 & 0.0410 & 0.9229 & 0.0680 & 0.0503 & 72.93 & 82.41 & \textbf{60.58} & 62.37 & 79.38 & 79.14 & 75.87 & 84.84 & 91.81 & 76.59 & 98.97\% \\
PO2(Split87) & \textbf{0.0308} & \textbf{0.9321} & \textbf{0.0515} & \textbf{0.0376} & 73.01 & 82.20 & 60.15 & \textbf{63.90} & 79.82 & 78.97 & 75.87 & 84.29 & 90.98 & 76.58 & 98.95\% \\
BOF4 & 0.0367 & 0.9258 & 0.0587 & 0.0437 & 73.72 & 81.82 & 59.98 & 63.79 & 78.84 & 78.94 & 76.09 & 84.29 & 91.05 & 76.50 & 98.85\% \\
NF4 & 0.0379 & 0.9253 & 0.0628 & 0.0455 & 72.61 & 82.03 & 60.41 & 62.74 & 79.00 & 79.07 & 75.43 & 85.40 & 91.36 & 76.45 & 98.79\% \\
PO2(NF4) & 0.0351 & 0.9285 & 0.0560 & 0.0419 & 73.72 & \textbf{82.74} & 60.24 & 62.39 & 78.89 & 78.58 & 75.90 & 83.55 & 91.58 & 76.40 & 98.72\% \\
SFP4 & 0.0402 & 0.9235 & 0.0667 & 0.0488 & 73.32 & 81.31 & 59.39 & 62.39 & 79.27 & 78.52 & 75.42 & 83.92 & 92.12 & 76.18 & 98.44\% \\
NVFP4 & 0.0441 & 0.9196 & 0.0727 & 0.0537 & 73.40 & 80.05 & 57.85 & 61.61 & 79.60 & 78.46 & 75.49 & 84.29 & 92.49 & 75.92 & 98.10\% \\
NVINT4 & 0.0613 & 0.9024 & 0.1000 & 0.0752 & 72.30 & 80.85 & 58.70 & 62.27 & 78.73 & 78.52 & 74.93 & 83.55 & \textbf{93.18} & 75.89 & 98.07\% \\
\bottomrule
\end{tabular}%
}
\end{table*}

\begin{table*}[t]
\centering
\caption{Weight and activation quantization (W4A4) on \textbf{Qwen3.5-2B}. KL columns report the KL divergence between the quantized and BF16 output distributions on the indicated calibration set ($\downarrow$); EAR$_{\text{top10},C4}$ (Expected Acceptance Rate), the maximum token-agreement probability under optimal coupling on C4 ($\uparrow$) Benchmark columns: Winogrande (WG), ARC-C, ARC-E, Lambada-standard (LMB), PIQA, Hellaswag 10-shot (HS10), MMLU, IFEval-Prompt (IFE-P), and GSM8K-CoT. \emph{Avg} is their mean and \emph{Recovery} is \emph{Avg}\,$\div$\,BF16. Best non-BF16 result per column is in \textbf{bold}.}
\label{tab:wa-qwen352b}
\setlength{\tabcolsep}{3pt}
\resizebox{\textwidth}{!}{%
\begin{tabular}{l cccc ccccccccc c c}
\toprule
 & \multicolumn{4}{c}{\textbf{Distribution-level}} & \multicolumn{9}{c}{\textbf{Benchmarks}} & & \\
\cmidrule(lr){2-5} \cmidrule(lr){6-14}
Method & KL$_{\text{t10},C4}\!\downarrow$ & EAR$_{\text{t10},C4}\!\uparrow$ & KL$_{\text{W2}}\!\downarrow$ & KL$_{\text{C4}}\!\downarrow$ & WG & ARC-C & ARC-E & LMB & PIQA & HS10 & MMLU & IFE-P & GSM8K & Avg & \textbf{Recovery} \\
\midrule
\textit{BF16 (baseline)} & \textit{0.0000} & \textit{1.0000} & \textit{0.0000} & \textit{0.0000} & \textit{64.09} & \textit{66.20} & \textit{41.21} & \textit{47.60} & \textit{72.47} & \textit{62.67} & \textit{59.29} & \textit{66.36} & \textit{70.20} & \textit{61.12} & \textit{100.00\%} \\
\midrule
Split87 & 0.0592 & 0.8956 & 0.0954 & 0.0766 & 61.96 & 61.28 & 39.76 & 41.49 & 71.22 & \textbf{61.01} & 52.94 & 61.55 & \textbf{65.81} & \textbf{57.45} & \textbf{93.99\%} \\
PO2(NF4) & 0.0546 & 0.8995 & 0.0874 & 0.0720 & 61.88 & 55.64 & 39.25 & 44.17 & \textbf{72.47} & 60.79 & \textbf{56.35} & 62.66 & 62.02 & 57.25 & 93.66\% \\
PO2(Split87) & \textbf{0.0504} & \textbf{0.9034} & \textbf{0.0814} & \textbf{0.0663} & 61.96 & 57.95 & 38.14 & 42.42 & 71.71 & 60.07 & 52.16 & \textbf{63.96} & 64.37 & 56.97 & 93.21\% \\
IF4 & 0.0663 & 0.8896 & 0.1056 & 0.0867 & 59.98 & 60.48 & 38.05 & 43.45 & 71.49 & 60.20 & 54.67 & 59.70 & 62.40 & 56.71 & 92.79\% \\
NF4 & 0.0651 & 0.8905 & 0.1055 & 0.0856 & 60.14 & \textbf{61.32} & 38.05 & 43.72 & 71.49 & 59.51 & 55.33 & 61.37 & 59.29 & 56.69 & 92.75\% \\
NVFP4 & 0.0782 & 0.8796 & 0.1327 & 0.1029 & 61.17 & 60.65 & \textbf{40.10} & 41.82 & 70.73 & 59.24 & 53.41 & 61.55 & 61.03 & 56.63 & 92.66\% \\
SFP4 & 0.0688 & 0.8872 & 0.1120 & 0.0905 & 62.04 & 57.37 & 39.08 & 42.38 & 71.49 & 59.86 & 53.16 & 61.55 & 60.20 & 56.35 & 92.19\% \\
4over6 & 0.0708 & 0.8860 & 0.1109 & 0.0931 & 60.30 & 58.00 & 36.69 & 40.73 & 70.78 & 59.63 & 55.63 & 59.52 & 64.82 & 56.23 & 92.00\% \\
BOF4 & 0.0647 & 0.8910 & 0.1102 & 0.0846 & 61.40 & 58.21 & 37.80 & 44.36 & 70.78 & 59.16 & 54.05 & 62.85 & 53.37 & 55.78 & 91.26\% \\
NVINT4 & 0.1075 & 0.8583 & 0.1744 & 0.1448 & 59.51 & 57.79 & 36.18 & 40.15 & 71.11 & 58.33 & 43.43 & 53.97 & 58.30 & 53.20 & 87.03\% \\
\bottomrule
\end{tabular}%
}
\end{table*}

\begin{table*}[t]
\centering
\caption{Weight and activation quantization (W4A4) on \textbf{Qwen3.5-4B}. KL columns report the KL divergence between the quantized and BF16 output distributions on the indicated calibration set ($\downarrow$); EAR$_{\text{top10},C4}$ (Expected Acceptance Rate), the maximum token-agreement probability under optimal coupling on C4 ($\uparrow$) Benchmark columns: Winogrande (WG), ARC-C, ARC-E, Lambada-standard (LMB), PIQA, Hellaswag 10-shot (HS10), MMLU, IFEval-Prompt (IFE-P), and GSM8K-CoT. \emph{Avg} is their mean and \emph{Recovery} is \emph{Avg}\,$\div$\,BF16. Best non-BF16 result per column is in \textbf{bold}.}
\label{tab:wa-qwen354b}
\setlength{\tabcolsep}{3pt}
\resizebox{\textwidth}{!}{%
\begin{tabular}{l cccc ccccccccc c c}
\toprule
 & \multicolumn{4}{c}{\textbf{Distribution-level}} & \multicolumn{9}{c}{\textbf{Benchmarks}} & & \\
\cmidrule(lr){2-5} \cmidrule(lr){6-14}
Method & KL$_{\text{t10},C4}\!\downarrow$ & EAR$_{\text{t10},C4}\!\uparrow$ & KL$_{\text{W2}}\!\downarrow$ & KL$_{\text{C4}}\!\downarrow$ & WG & ARC-C & ARC-E & LMB & PIQA & HS10 & MMLU & IFE-P & GSM8K & Avg & \textbf{Recovery} \\
\midrule
\textit{BF16 (baseline)} & \textit{0.0000} & \textit{1.0000} & \textit{0.0000} & \textit{0.0000} & \textit{69.46} & \textit{75.67} & \textit{54.10} & \textit{58.65} & \textit{78.40} & \textit{74.83} & \textit{74.31} & \textit{82.26} & \textit{91.66} & \textit{73.26} & \textit{100.00\%} \\
\midrule
SFP4 & 0.0528 & 0.9047 & 0.1084 & 0.0681 & \textbf{69.22} & 75.29 & 53.84 & 55.40 & 76.99 & 73.05 & 71.73 & 78.19 & 89.54 & \textbf{71.47} & \textbf{97.56\%} \\
NF4 & 0.0490 & 0.9083 & 0.1054 & 0.0637 & 67.72 & 75.13 & 52.39 & 55.42 & 77.20 & 73.17 & 72.03 & 78.00 & 88.93 & 71.11 & 97.07\% \\
IF4 & 0.0471 & 0.9097 & 0.1125 & 0.0616 & 67.80 & 73.91 & 51.88 & 54.82 & 76.61 & 73.11 & 71.96 & 77.26 & 89.54 & 70.76 & 96.60\% \\
PO2(Split87) & \textbf{0.0376} & \textbf{0.9195} & \textbf{0.0869} & \textbf{0.0487} & 67.56 & 72.22 & 50.51 & 54.80 & 76.93 & 73.57 & 72.56 & \textbf{78.56} & 89.69 & 70.71 & 96.52\% \\
PO2(NF4) & 0.0406 & 0.9166 & 0.0919 & 0.0525 & 67.40 & 72.56 & 50.85 & 54.90 & \textbf{77.75} & 73.42 & \textbf{72.72} & 76.71 & 89.76 & 70.68 & 96.48\% \\
BOF4 & 0.0460 & 0.9109 & 0.1030 & 0.0597 & 66.38 & 72.64 & 50.43 & 55.39 & 77.09 & \textbf{73.63} & 71.76 & 78.56 & 90.14 & 70.67 & 96.47\% \\
4over6 & 0.0516 & 0.9059 & 0.1112 & 0.0670 & 67.80 & 74.66 & 51.45 & \textbf{56.80} & 76.33 & 73.24 & 70.86 & 78.00 & 86.58 & 70.64 & 96.42\% \\
Split87 & 0.0438 & 0.9132 & 0.0957 & 0.0571 & 67.56 & 72.10 & 49.66 & 56.47 & 76.66 & 73.18 & 71.94 & 78.56 & 88.17 & 70.48 & 96.20\% \\
NVINT4 & 0.0714 & 0.8872 & 0.1558 & 0.0957 & 65.59 & \textbf{76.85} & \textbf{55.03} & 52.98 & 76.71 & 71.60 & 70.79 & 75.60 & 87.64 & 70.31 & 95.98\% \\
NVFP4 & 0.0586 & 0.8993 & 0.1305 & 0.0767 & 66.54 & 71.84 & 49.74 & 54.08 & 76.61 & 72.77 & 71.40 & 77.45 & 87.04 & 69.72 & 95.17\% \\
\bottomrule
\end{tabular}%
}
\end{table*}

\begin{table*}[t]
\centering
\caption{Weight and activation quantization (W4A4) on \textbf{Qwen3.5-9B}. KL columns report the KL divergence between the quantized and BF16 output distributions on the indicated calibration set ($\downarrow$); EAR$_{\text{top10},C4}$ (Expected Acceptance Rate), the maximum token-agreement probability under optimal coupling on C4 ($\uparrow$) Benchmark columns: Winogrande (WG), ARC-C, ARC-E, Lambada-standard (LMB), PIQA, Hellaswag 10-shot (HS10), MMLU, IFEval-Prompt (IFE-P), and GSM8K-CoT. \emph{Avg} is their mean and \emph{Recovery} is \emph{Avg}\,$\div$\,BF16. Best non-BF16 result per column is in \textbf{bold}.}
\label{tab:wa-qwen359b}
\setlength{\tabcolsep}{3pt}
\resizebox{\textwidth}{!}{%
\begin{tabular}{l cccc ccccccccc c c}
\toprule
 & \multicolumn{4}{c}{\textbf{Distribution-level}} & \multicolumn{9}{c}{\textbf{Benchmarks}} & & \\
\cmidrule(lr){2-5} \cmidrule(lr){6-14}
Method & KL$_{\text{t10},C4}\!\downarrow$ & EAR$_{\text{t10},C4}\!\uparrow$ & KL$_{\text{W2}}\!\downarrow$ & KL$_{\text{C4}}\!\downarrow$ & WG & ARC-C & ARC-E & LMB & PIQA & HS10 & MMLU & IFE-P & GSM8K & Avg & \textbf{Recovery} \\
\midrule
\textit{BF16 (baseline)} & \textit{0.0000} & \textit{1.0000} & \textit{0.0000} & \textit{0.0000} & \textit{73.56} & \textit{74.24} & \textit{55.89} & \textit{63.19} & \textit{79.92} & \textit{78.70} & \textit{78.67} & \textit{83.55} & \textit{93.03} & \textit{75.64} & \textit{100.00\%} \\
\midrule
PO2(Split87) & \textbf{0.0310} & \textbf{0.9278} & \textbf{0.1164} & \textbf{0.0402} & 72.69 & 76.43 & 56.23 & \textbf{62.45} & 80.14 & 78.27 & 76.99 & \textbf{84.29} & 91.96 & \textbf{75.50} & \textbf{99.81\%} \\
BOF4 & 0.0381 & 0.9205 & 0.1410 & 0.0491 & 71.59 & 76.77 & 57.08 & 62.08 & 80.09 & 78.18 & 76.63 & 83.18 & 91.74 & 75.26 & 99.50\% \\
PO2(NF4) & 0.0350 & 0.9233 & 0.1377 & 0.0456 & 72.61 & 74.79 & 55.55 & 62.24 & \textbf{80.36} & \textbf{78.56} & \textbf{77.10} & 82.62 & \textbf{92.42} & 75.14 & 99.34\% \\
NF4 & 0.0388 & 0.9194 & 0.1470 & 0.0507 & 72.06 & \textbf{78.41} & \textbf{57.17} & 61.50 & 79.65 & 77.72 & 76.36 & 80.96 & 91.81 & 75.07 & 99.25\% \\
IF4 & 0.0400 & 0.9184 & 0.1427 & 0.0519 & \textbf{73.56} & 75.34 & 56.14 & 61.85 & 79.49 & 78.22 & 76.76 & 82.62 & 91.36 & 75.04 & 99.21\% \\
4over6 & 0.0416 & 0.9168 & 0.1416 & 0.0545 & 72.85 & 76.05 & 56.83 & 59.67 & 79.33 & 77.28 & 76.41 & 83.73 & 91.74 & 74.88 & 98.99\% \\
Split87 & 0.0349 & 0.9237 & 0.1274 & 0.0449 & 72.93 & 75.04 & 55.03 & 61.17 & 79.27 & 78.08 & 76.47 & 82.99 & 91.81 & 74.76 & 98.83\% \\
NVFP4 & 0.0451 & 0.9135 & 0.1497 & 0.0586 & 72.06 & 75.46 & 55.12 & 62.04 & 79.22 & 77.68 & 75.95 & 82.99 & 91.74 & 74.70 & 98.75\% \\
SFP4 & 0.0406 & 0.9177 & 0.1415 & 0.0526 & 71.74 & 73.57 & 54.01 & 60.33 & 78.62 & 77.36 & 75.59 & 84.29 & 90.90 & 74.05 & 97.89\% \\
NVINT4 & 0.0605 & 0.8984 & 0.1888 & 0.0800 & 71.19 & 75.25 & 55.38 & 60.55 & 79.33 & 76.70 & 75.39 & 80.96 & 89.99 & 73.86 & 97.65\% \\
\bottomrule
\end{tabular}%
}
\end{table*}

\begin{table*}[t]
\centering
\caption{Weight and activation quantization (W4A4) on \textbf{Qwen3.5-27B}. KL columns report the KL divergence between the quantized and BF16 output distributions on the indicated calibration set ($\downarrow$); EAR$_{\text{top10},C4}$ (Expected Acceptance Rate), the maximum token-agreement probability under optimal coupling on C4 ($\uparrow$) Benchmark columns: Winogrande (WG), ARC-C, ARC-E, Lambada-standard (LMB), PIQA, Hellaswag 10-shot (HS10), MMLU, IFEval-Prompt (IFE-P), and GSM8K-CoT. \emph{Avg} is their mean and \emph{Recovery} is \emph{Avg}\,$\div$\,BF16. Best non-BF16 result per column is in \textbf{bold}.}
\label{tab:wa-qwen3527b}
\setlength{\tabcolsep}{3pt}
\resizebox{\textwidth}{!}{%
\begin{tabular}{l cccc ccccccccc c c}
\toprule
 & \multicolumn{4}{c}{\textbf{Distribution-level}} & \multicolumn{9}{c}{\textbf{Benchmarks}} & & \\
\cmidrule(lr){2-5} \cmidrule(lr){6-14}
Method & KL$_{\text{t10},C4}\!\downarrow$ & EAR$_{\text{t10},C4}\!\uparrow$ & KL$_{\text{W2}}\!\downarrow$ & KL$_{\text{C4}}\!\downarrow$ & WG & ARC-C & ARC-E & LMB & PIQA & HS10 & MMLU & IFE-P & GSM8K & Avg & \textbf{Recovery} \\
\midrule
\textit{BF16 (baseline)} & \textit{0.0000} & \textit{1.0000} & \textit{0.0000} & \textit{0.0000} & \textit{80.66} & \textit{79.88} & \textit{61.35} & \textit{71.36} & \textit{81.99} & \textit{79.70} & \textit{84.49} & \textit{87.25} & \textit{95.68} & \textit{80.26} & \textit{100.00\%} \\
\midrule
IF4 & 0.0215 & 0.9416 & 0.1101 & 0.0274 & 79.24 & 82.53 & \textbf{62.71} & 70.21 & 82.43 & 79.50 & 83.75 & 86.88 & \textbf{96.29} & \textbf{80.39} & \textbf{100.16\%} \\
NVFP4 & 0.0264 & 0.9358 & 0.1208 & 0.0335 & 79.56 & \textbf{82.79} & 62.20 & 70.33 & 82.37 & 78.84 & 84.16 & 87.80 & 95.30 & 80.37 & 100.14\% \\
NF4 & 0.0225 & 0.9404 & 0.1065 & 0.0286 & 79.24 & 82.32 & 61.69 & 70.15 & 82.54 & 79.29 & 84.03 & 88.17 & 95.75 & 80.35 & 100.12\% \\
PO2(Split87) & \textbf{0.0173} & \textbf{0.9476} & \textbf{0.0992} & \textbf{0.0219} & 79.01 & 82.20 & 61.77 & 70.54 & 81.83 & 78.47 & \textbf{84.47} & \textbf{88.54} & 95.83 & 80.30 & 100.04\% \\
4over6 & 0.0234 & 0.9394 & 0.1092 & 0.0297 & 79.72 & 82.03 & 62.37 & 70.04 & 81.83 & 79.73 & 82.98 & 87.43 & 96.13 & 80.25 & 99.99\% \\
SFP4 & 0.0236 & 0.9392 & 0.1084 & 0.0299 & 80.51 & 82.45 & 61.09 & 70.39 & 82.43 & 79.55 & 83.91 & 85.58 & 95.91 & 80.20 & 99.92\% \\
Split87 & 0.0199 & 0.9440 & 0.1061 & 0.0252 & 80.19 & 81.82 & 61.86 & 69.38 & 81.94 & 78.65 & 84.01 & 87.43 & 96.13 & 80.16 & 99.87\% \\
BOF4 & 0.0211 & 0.9421 & 0.1074 & 0.0269 & 79.01 & 81.31 & 62.37 & 70.68 & 82.05 & 78.99 & 83.65 & 87.25 & 95.83 & 80.13 & 99.83\% \\
NVINT4 & 0.0321 & 0.9281 & 0.1378 & 0.0417 & 79.48 & 81.61 & 61.26 & 69.01 & 82.21 & \textbf{80.24} & 82.87 & 87.06 & 96.13 & 79.99 & 99.66\% \\
PO2(NF4) & 0.0189 & 0.9452 & 0.1010 & 0.0239 & 80.35 & 81.73 & 60.32 & 69.09 & 81.99 & 78.43 & 83.64 & 87.99 & 95.98 & 79.95 & 99.61\% \\
\bottomrule
\end{tabular}%
}
\end{table*}

\end{document}